\documentclass[11pt]{article}

\usepackage{microtype}
\usepackage{graphicx}
\usepackage{subfigure}
\usepackage{amsmath,amsbsy,amsfonts,amssymb,amsthm,dsfont,bm}
\usepackage{algorithm,algorithmic,mathtools}
\usepackage{color,cases,multirow}
\usepackage[top=1.0in, bottom=1.0in, left=1.2in, right=1.2in]{geometry}
\usepackage{hyperref}
 \hypersetup{
     colorlinks=true,
     linkcolor=blue,
     filecolor=blue,
     citecolor = blue,
     urlcolor=cyan,
}

\usepackage{authblk}
\usepackage[math]{cellspace}
\newtheorem{theorem}{Theorem}[section]
\newtheorem{proposition}[theorem]{Proposition}
\newtheorem{lemma}[theorem]{Lemma}

\newtheorem{assumption}[theorem]{Assumption}
\newtheorem{corollary}[theorem]{Corollary}
\theoremstyle{definition}
\newtheorem{definition}{Definition}
\theoremstyle{remark}
\newtheorem{remark}{Remark}

\numberwithin{equation}{section}

\newcommand\Def{\stackrel{\textrm{def}}{=}}

\newcommand{\poly}{\text{poly}}
\newcommand{\half}{\frac{1}{2}}
\newcommand{\hcR}{\hat{\mathcal{R}}_n}

\newcommand{\tf}{\tilde{f}}

\newcommand{\tB}{\tilde{B}}

\newcommand{\RR}{\mathbb{R}}
\newcommand{\EE}{\mathbb{E}}
\newcommand{\PP}{\mathbb{P}}
\newcommand{\NN}{\mathbb{N}}
\renewcommand{\SS}{\mathbb{S}}

\newcommand{\cE}{\mathcal{E}}
\newcommand{\cF}{\mathcal{F}}
\newcommand{\cH}{\mathcal{H}}
\newcommand{\cI}{\mathcal{I}}

\newcommand{\cN}{\mathcal{N}}

\newcommand{\cP}{\mathcal{P}}
\newcommand{\cR}{\mathcal{R}}

\newcommand{\ba}{\bm{a}}
\newcommand{\bb}{\bm{b}}

\newcommand{\be}{\bm{e}}
\newcommand{\bh}{\bm{h}}
\newcommand{\bg}{\bm{g}}
\newcommand{\bx}{\bm{x}}

\newcommand{\bz}{\bm{z}}

\newcommand{\bv}{\bm{v}}
\newcommand{\bw}{\bm{w}}
\newcommand{\br}{\bm{r}}

\newcommand{\alphal}{\alpha^{(l)}}
\newcommand{\betal}{\bm{\beta}^{(l)}}
\newcommand{\gammal}{\bm{\gamma}^{(l)}}
\newcommand{\al}{\ba^{(l)}}
\newcommand{\rl}{\bm{r}^{(l)}}
\newcommand{\Cl}{C^{(l)}}
\newcommand{\Bl}{B^{(l)}}

\newcommand{\yl}{y^{(l)}}
\newcommand{\zl}{\bm{z}^{(l)}}
\newcommand{\gl}{\bg^{(l)}}

\newcommand{\ol}{o^{(l)}}
\newcommand{\ta}{\tilde{\ba}}
\newcommand{\rad}{\text{Rad}}
\newcommand{\gen}{\text{gen}}

\title{Analysis of the Gradient Descent Algorithm \\
        for a Deep Neural Network Model with Skip-connections}

\author[1,2,3]{Weinan E \thanks{weinan@math.princeton.edu}}
\author[2]{Chao Ma, \thanks{cham@princeton.edu}}
\author[2]{Qingcan Wang, \thanks{qingcanw@princeton.edu}}
\author[2]{Lei Wu \thanks{leiwu@princeton.edu}}

\affil[1]{Department of Mathematics, Princeton University}
\affil[2]{Program in Applied and Computational Mathematics, Princeton University}
\affil[3]{Beijing Institute of Big Data Research}

\date{}

\begin{document}
\maketitle
\begin{abstract}
The behavior of the gradient descent (GD) algorithm is analyzed for a deep neural network model with skip-connections. It is proved that in the over-parametrized regime, for a suitable initialization,  with high probability GD can find a global minimum exponentially fast. Generalization error estimates along the GD path are also established. As a consequence, it is shown that when the target function is in  the reproducing kernel Hilbert space (RKHS) with a kernel defined by the initialization,  there exist generalizable early-stopping solutions along the GD path. In addition, it is also shown that the GD path is uniformly close to the functions given by the related random feature model.  Consequently, in this ``implicit regularization'' setting, the deep neural network model deteriorates to a random feature model. Our results hold for neural networks of any width larger than the input  dimension.
\end{abstract}

\section{Introduction}
This paper is concerned with the following questions on the gradient descent (GD) algorithm for deep neural network models:
\begin{enumerate}
\item Under what condition, can the algorithm find a global minimum of the empirical risk?
\item Under what condition, can the algorithm find models that generalize, without using any explicit regularization?
\end{enumerate}
These questions are addressed for a specific deep neural network model with skip connections.
For the first question, it is shown that with proper initialization, the gradient descent algorithm converges to a global minimum exponentially fast, as long as the network is deep enough.
For the second question, it is shown that if in addition the target function belongs to a certain reproducing kernel
Hilbert space (RKHS) with kernel defined by the initialization, then the  gradient descent algorithm
does find models that can generalize. This result is obtained as a consequence of the estimates on the
 the generalization error along the GD path.
 However, it is also shown that the GD path is uniformly close to functions generated by the GD path for the
 related random feature model.
 Therefore in this particular setting, as far as ``implicit regularization'' is concerned,  this deep neural network model is no better than
 the random feature model.



In recent years there has been a great deal of interest  on the two questions raised above\cite{daniely2017sgd,du2018gradient,du2018deepgradient,arora2019fine,allen2018convergence,allen2018learning,zou2018stochastic,cao2019generalization,jacot2018neural,li2018learning,song2018mean,rotskoff2018parameters,chizat2018global,sirignano2018mean, bartlett2018gradient,arora2018a,du2019width, zhang2016understanding,neyshabur2014search}.
 An important recent advance is the realization that over-parametrization can simplify the analysis of GD dynamics in two ways:  The first is that in the over-parametrized regime, the parameters do not have to change much in order to make an $O(1)$ change to the function that they represent \cite{daniely2017sgd,li2018learning}.  This gives rise to the possibility that only a local analysis in the neighborhood of the initialization is necessary in order to analyze the GD algorithm. The second is that over-parametrization can improve the non-degeneracy of
 the associated Gram matrix~\cite{xie2016diverse}, thereby ensuring exponential convergence of the GD algorithm \cite{du2018gradient}.

 Using these ideas, \cite{allen2018convergence,zou2018stochastic,du2018deepgradient,du2018gradient} proved that (stochastic) gradient descent converges to a global minimum of the empirical risk with an exponential rate. \cite{jacot2018neural} showed that in the infinite-width limit, the GD dynamics for deep fully connected neural networks with Xavier initialization can be characterized by a fixed neural tangent kernel.
 \cite{daniely2017sgd,li2018learning} considered the online learning setting and proved that stochastic gradient descent can achieve a population error of $\varepsilon$ using $\poly(1/\varepsilon)$  samples.
\cite{cao2019generalization} proved that GD can find the generalizable solutions when the target function comes from certain RKHS.
 These results all share one thing in common:  They all
require that the network width $m$ satisfies  $m\geq \poly(n,L)$, where $L, n$ denote the network depth and  training set size, respectively. In fact, \cite{daniely2017sgd,cao2019generalization} required that $m\geq \poly(n,2^L)$.
In other words, these results are concerned with very wide networks.
In contrast, in this paper, we will focus on  deep networks with fixed width (assumed to be larger than $d$ where
$d$ is the input dimension).

\subsection{The motivation}
Our work is motivated strongly by the results of the companion paper \cite{e2019twolayer} in which similar questions were addressed
for the two-layer neural network model. It was proved in \cite{e2019twolayer} that in the so-called ``implicit regularization'' setting, the GD dynamics for the two-layer neural network model is closely approximated by the GD dynamics for a random feature model with
the features defined by the initialization.  For over-parametrized models, this statement is valid uniformly for all time.
In the general case, this statement is valid at least for finite time intervals during which early stopping leads to generalizable
models for target functions in the relevant reproducing kernel Hilbert space (RKHS).
The numerical results reported in \cite{e2019twolayer} nicely corroborated these theoretical findings.

To understand what happens for deep neural network models, we first turn to the ResNet model:
 \begin{equation}
\label{ResNet}
 \bh^{(l+1)} = \bh^{(l)}+ U^{(l)}\sigma(V^{(l)} \bh^{(l)}), \quad l=0, 1, \cdots, L-1
 \end{equation}
 \[
 \bh^{(0)} =  (\bx^T,1)^T\in\RR^{d+1}, \quad f_L(\bx, \theta) = \bw^T \bh^{(L)}
 \]
 where $U^{(l)}\in\RR^{(d+1)\times m}, V^{(l)}\in \RR^{m\times (d+1)}, \bw=(0,\dots,0,1)$, and $\theta=\{U^{(l)},V^{(l)}\}_{l=1}^L$ denote the all the parameters to be trained.
 
 A main observation exploited in \cite{e2019twolayer} is the time scale separation between the GD dynamics for the coefficients
in- and outside the activation function, i.e. the $\{U^{(l)}\}$'s and the $\{V^{(l)}\}$'s
In a typical practical setting, one would initialize the $\{V^{(l)}\}$'s to be $O(1)$ and the $\{U\}$'s to be $o(1)$.
This results in a slower dynamics for the $\{V^{(l)}\}$'s, compared with the dynamics for the $\{U^{(l)}\}$'s, due to
the factor presence of an extra factor of $U^{(l)}$ in the dynamical equation for $V^{(l)}$.
In the case of two-layer networks, this separation of time scales resulted in the fact that the parameters inside the activation function
were effectively frozen during the time period of interest.  Therefore the GD path stays close to the GD path for the random
feature model with the features given by the initialization. 

To see whether similar things happen for the ResNet model, we consider the following ``compositional random feature model''
in which \eqref{ResNet} is replaced by
 \begin{equation}
 \label{Crfm}
 \bh^{(l+1)} = \bh^{(l)}+U^{(l)}\sigma(V^{(l)}(0) \bh^{(l)}), \quad l=0, 1, \cdots, L-1
 \end{equation}
Note that in  \eqref{Crfm}  the $V^{(l)}$'s are fixed at their initial values,
the only parameters to be updated by the GD dynamics are the $U^{(l)}$'s. 

\paragraph*{Numerical experiments}
Here we provide numerical evidences for the above intuition by 
considering a very simple target function:
$
f^*(\bx) = \max(x_1, 0),
$
where $\bx=(x_1, \cdots, x_d)^T$.   We initialize \eqref{ResNet} and \eqref{Crfm} by $U_{i,k}=0, V_{k,j}\sim \mathcal{N}(0,1/m), \forall \, k\in [m], i,j\in [d+1]$. Since we are interested in the effect of depth, we choose $m=1$. Please refer to Appendix~\ref{sec: res-setup} for more details.

Figure \ref{fig: optimization} displays the comparison of the GD dynamics for ResNet and the related ``compositional random feature model''. 
We see a clear indication that 
(1) GD algorithm  converges to a global minimum of the empirical risk for deep residual network, and (2) for deep neural
networks, the GD dynamics for the two models stays close to each other.

\begin{figure}[!h]
\centering 
\includegraphics[width=0.31\textwidth]{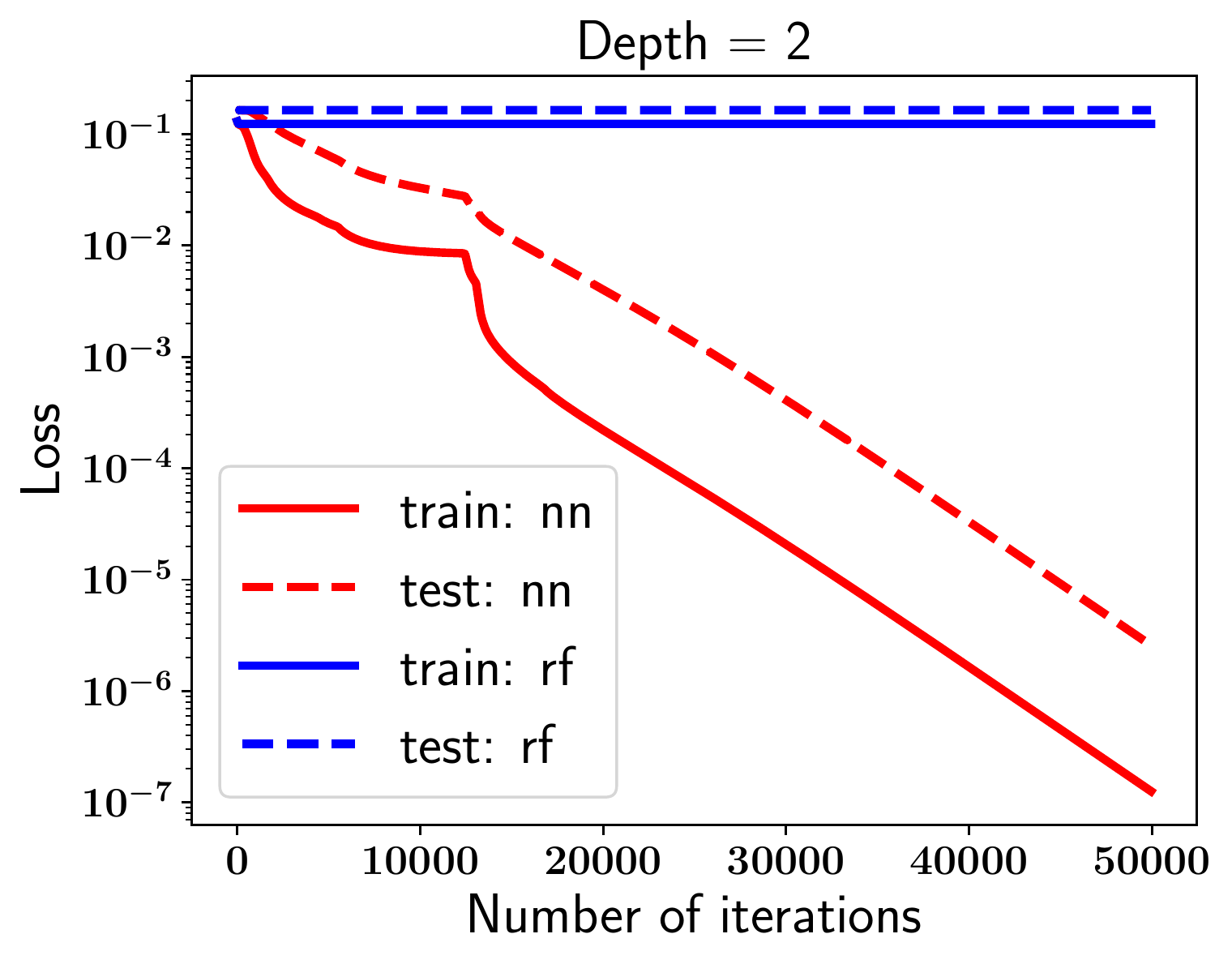}
\includegraphics[width=0.31\textwidth]{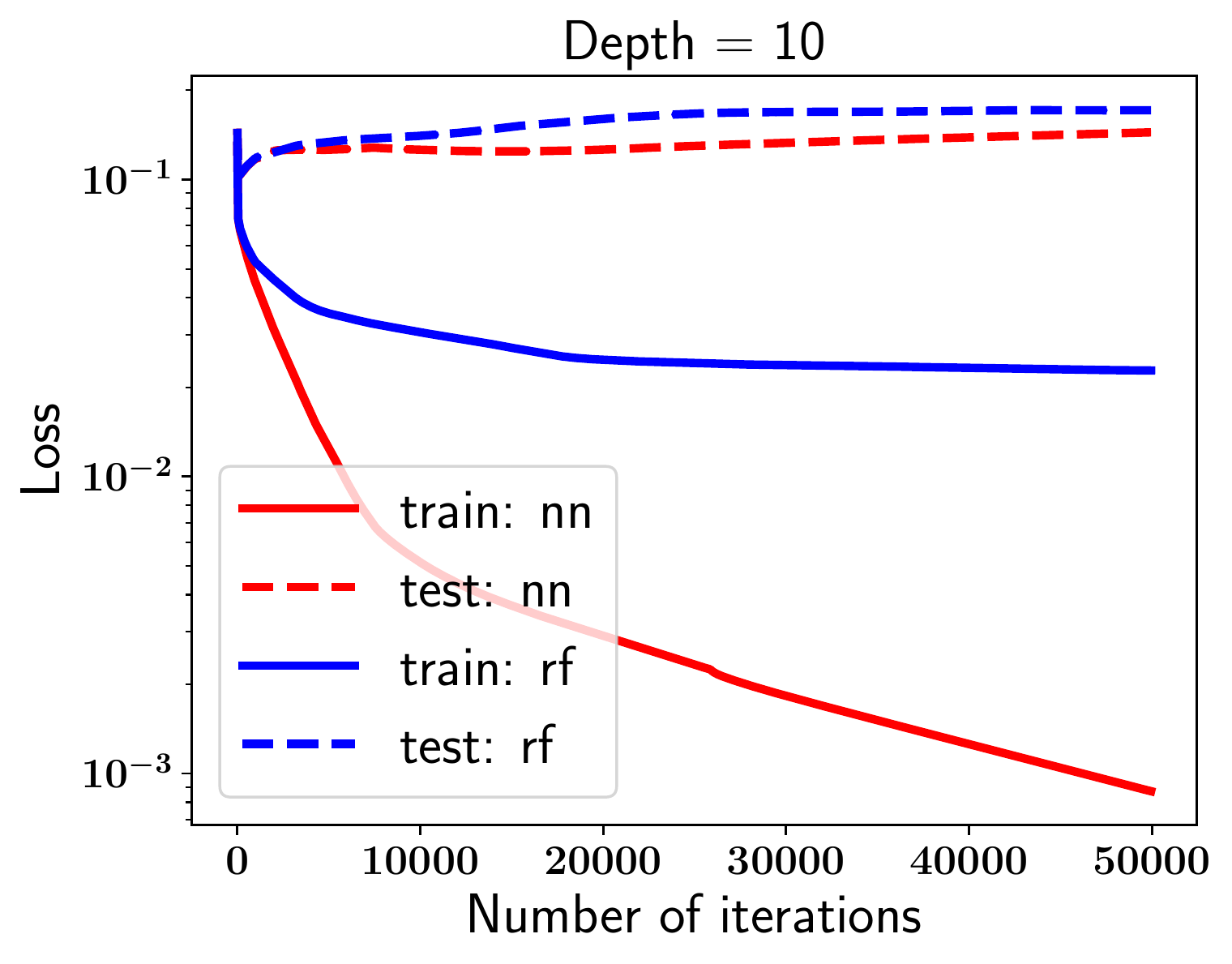}
\includegraphics[width=0.31\textwidth]{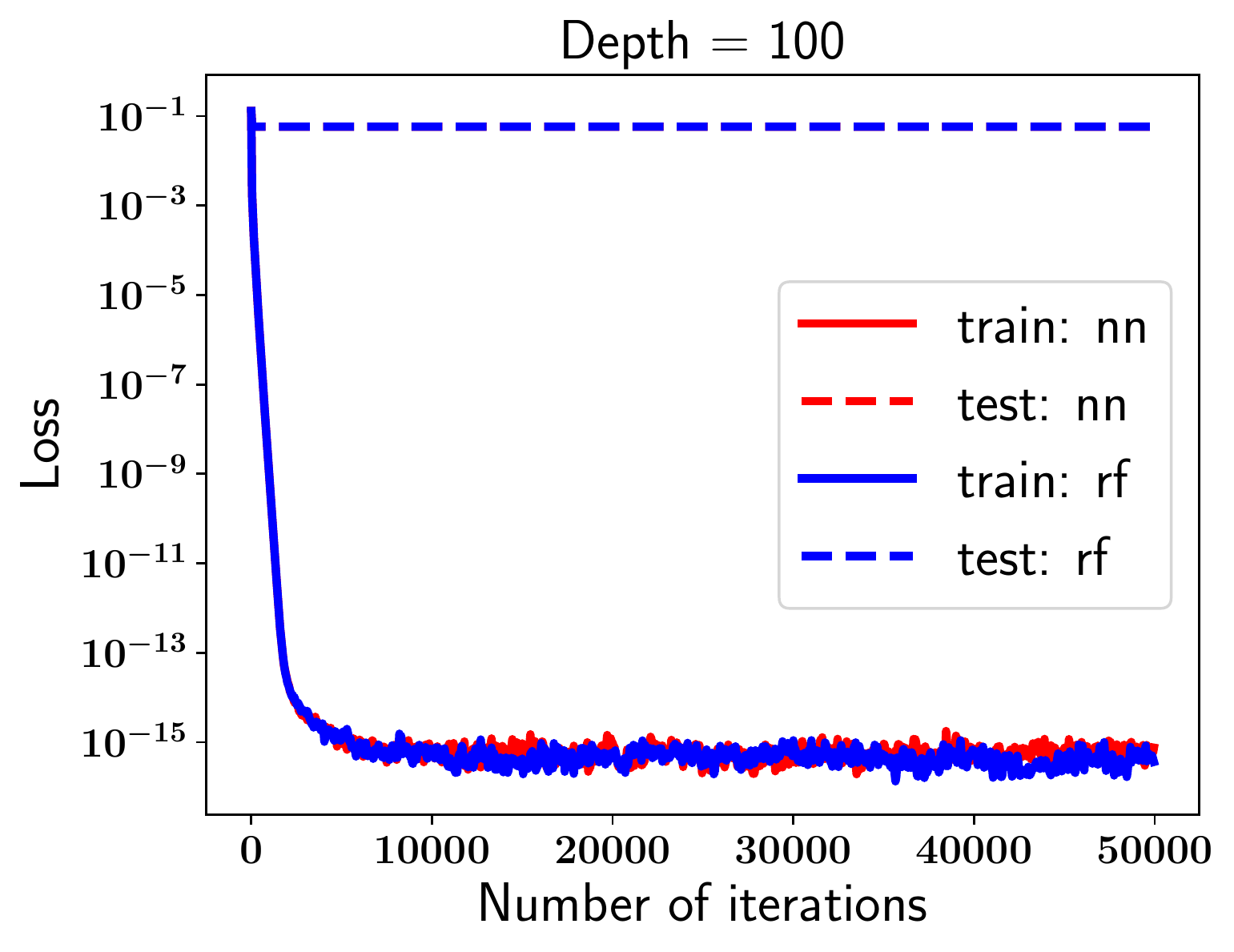}
\caption{The time history of the GD algorithm for the residual network (\texttt{nn}) and compositional random feature (\texttt{rf}). Left: depth=2; Middle: depth=10; Right: depth=100. Here we use the same learning rate for two models. }
\label{fig: optimization}
\end{figure}

\begin{figure}[!h]
\centering 
\includegraphics[width=0.4\textwidth]{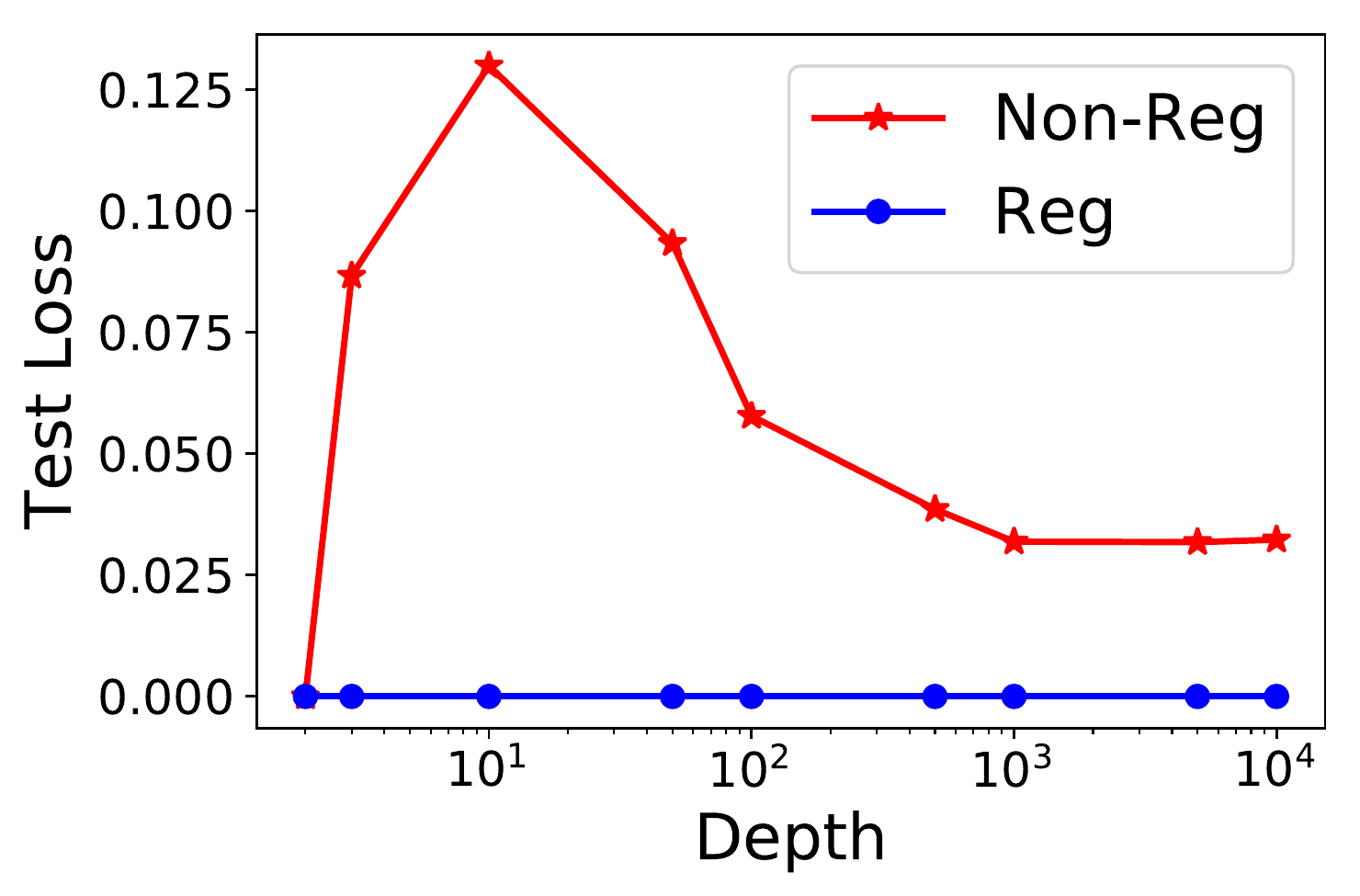}
\caption{Comparison of test accuracy of the regularized and the un-regularized residual networks with different depths. For the regularized model, we choose $\lambda=0.005$.}
\label{fig: generalization}
\end{figure}

Figure \ref{fig: generalization} shows the testing error for the optimal (convergent) solution shown in Figure \ref{fig: optimization}
as the depth of the ResNet changes.  We see that the testing error seems to be settling down on a finite value
as the network depth is increased.  As a comparison, we also show the testing error for the optimizers of the regularized model proposed in 
\cite{ma2019priori} (see \eqref{eqn: path-norm-reg}). 
One can see that for this particular target function, the testing error for the minimizers of the regularized model
is consistently very small as one varies the depth of the network.

These results are similar to the ones shown in \cite{e2019twolayer} for two-layer neural networks.
They suggest that  for ResNets, GD algorithm is able to find global minimum for the empirical risk but in terms of
the generalization property,  the resulting model may not be better than the compositional random feature model.

On the theoretical side, we have not yet succeeded in dealing directly with the ResNet model.
Therefore in this paper  we will deal  instead with a modified model which shares a lot of common features
with the ResNet model but the simplifies the task of analyzing error propagation between the layers.
We believe that the insight we have gained in this analysis is helpful for understanding general deep network models.

\subsection{Our contribution}
In this paper, we analyze the gradient descent algorithm for a particular class of
deep neural networks with skip-connections.
We consider the least square loss and assume that the nonlinear activation function  is Lipschitz continuous (e.g. $\texttt{Tanh, ReLU}$).
\begin{itemize}
\item
We prove that if the depth $L$ satisfies $L\geq \poly(n)$,  then gradient descent  converges to a global minimum with zero training error at an exponential rate.
This result is proved by only assuming that the network width is larger than $d$.
As noted above, the previous optimization results \cite{du2018deepgradient,allen2018convergence,zou2018stochastic} require that the width $m$ satisfies $m\geq \poly(n,L)$.


\item  We  provide a general estimate for the generalization error along the GD path, assuming that the target function is in a RKHS  with the kernel defined by the initialization.
As a consequence, we show that population risk is  bounded from above by $O(1/\sqrt{n})$ if
certain early stopping rules are used.   In contrast, the generalization result in \cite{cao2019generalization} requires that $m\geq \poly(n,2^L)$.

\item We prove that the GD path is uniformly close to the functions given by
the related random feature model (see Theorem \ref{thm: uniform-closeness}). Consequently the  generalization property of the resulting model is no better than that of the random feature model.  This allows us to conclude that in this ``implicit regularization'' setting,
the deep neural network model deteriorates to a random feature model.
In contrast, it has been established in \cite{ma2018priori,ma2019priori} that for suitable explicitly regularized models, optimal generalization error estimates (e.g. rates comparable to the Monte Carlo rate) can be proved for a much larger class of target functions.
\end{itemize}
These results are very much analogous to the ones proved in \cite{e2019twolayer}  for two-layer neural networks.

One main technical ingredient in this work is to use a combination of the identity mapping and skip-connections to stabilize the forward and backward propagation in the neural network. This enable us to consider deep neural networks with fixed width.
The second main ingredient is the exploitation of a possible 
time scale separation between the GD dynamics for the parameters in- and outside the 
activation function: The parameters inside the activation function are effectively frozen during the GD dynamics compared with
the parameters outside the activation function.

\section{Preliminaries}
 Throughout this paper, we let $[n]=\{1,2,\dots,n\}$, and use $\|\|$ and $\|\|_F$ to denote the $\ell_2$ and Frobenius norms, respectively.  For a matrix $A$, we use $A_{i,:}, A_{:,j}, A_{i,j}$  to denote its $i$-th row, $j$-th column and $(i, j)$-th entry, respectively. We let $\SS^{d-1}=\{\bx \in \RR^d\,:\,\|\bx\|=1$ and use $\pi_0$ to indicate the uniform distribution over $\SS^{d-1}$.  We use $X\lesssim Y$ as a shorthand notation for  $X\leq C Y$, where $C$ is some absolute constant. $X\gtrsim Y$ is similarly defined.

\subsection{Problem setup}
We consider the regression problem with training data set given by  $\{(\bx_i,y_i)\}_{i=1}^n$, where $\{\bx_i\}_{i=1}^n$ are i.i.d. samples drawn from a fixed (but unknown) distribution $\rho$.
For simplicity, we assume $\|\bx\|_2= 1$ and $|y|\leq 1$. We use $f(\cdot;\Theta): \RR^{d} \to \RR$ to denote the model with parameter $\Theta$.
We are interested in minimizing the empirical risk, defined by
\begin{equation}
    \hat{\cR}_n(\Theta)  = \frac{1}{2n}\sum_{i=1}^n (f(\bx_i;\Theta-y_i)^2.
\end{equation}
We let  $e(\bx,y)= f(\bx;\Theta)-y$ and $\hat{\be}=(e(\bx_1,y_1),\dots,e(\bx_n,y_n))^T\in\RR^n$, then $\hat{\cR}_n(\Theta)=\frac{1}{2n}\|\hat{\be}\|^2$.

For the generalization problem, we need to specify how the $\{y_j\}_{j=1}^n$'s are obtained.
Let $f^*: \RR^d \rightarrow \RR$ be our target function. Then we have $y_j = f^*(\bx_j)$. We will assume that  there
are no measurement noises. This makes the argument more transparent but does not change things qualitatively:
Essentially the same argument applies to the case with measurement noise.

Our goal is to estimate the population risk, defined by
\[
    \cR(\Theta) = \frac{1}{2}\EE_{\rho}[(f(\bx;\Theta)-f^*(\bx))^2]
\]

\subsubsection*{Deep neural networks with skip-connections}
We will consider a special deep neural network model with multiple skip-connections, defined by
\begin{equation}\label{eqn: net-arch}
\begin{aligned}
\bh^{(1)} &= (\bx,0)^T\in\RR^{d+1} \\
\bh^{(l+1)} &=
{
\begin{pmatrix}
\bh^{(1)}_{1:d}\\
\bh^{(l)}_{d+1}
\end{pmatrix}  +  U^{(l)} \sigma(V^{(l)}\bh^{(l)}) }, \quad l=1, \cdots, L-1\\
f(\bx;\Theta) &= \bw^T \bh^{(L)}.
\end{aligned}
\end{equation}
 Here $U^{(l)}\in \RR^{(d+1)\times m}, V^{(l)}\in \RR^{m\times (d+1)},\bw\in\RR^{d+1}$. Note that
  $L$ and $\max\{m,d+1\}$ are the depth and width of the network respectively.
 $\sigma(\cdot):\RR\to\RR$ is a scalar nonlinear activation function, which is assumed to be 1-Lipschitz continuous and $\sigma(0)=0$. { For any vector $\bv\in\RR^{N}$ we define $\sigma(\bv)=(\sigma(v_1),\dots,\sigma(v_N))^T$ .}
 For simplicity, we fix $\bw$  to be
 $(0,\dots,0,1)^T$. Thus the parameters that need to be estimated are: $\Theta=\{U^{(l)},V^{(l)}\}_{l=1}^{L}$.
  We also define $\bg^{(l)}=\sigma(V^{(l)}\bh^{(l)})\in\RR^{m}$, the output of the $l$-th nonlinear hidden layer.

This network model has the following feature:  The first $d$ entries of $\bh^{(l)}$ are directly connected to the input layer by  a long-distance skip-connection, only the last entry is connected to the previous layer.
As will be seen  later, the long-distance skip-connections help to stabilize the deep network.
We further let:
\[
U^{(l)}=\begin{pmatrix}
        C^{(l)}\\
        (\ba^{(l)})^T
    \end{pmatrix},\,\,
V^{(l)}=\begin{pmatrix}
        B^{(l)} & \br^{(l)}
    \end{pmatrix},\,\,
\bh^{(l)} = \begin{pmatrix}
\bz^{(l)}\\
 y^{(l)}
\end{pmatrix},
\]
where $ C^{(l)} \in \RR^{d\times m}, \ba^{(l)}\in \RR^m$, $B^{(l)}\in\RR^{m\times d}, \br^{(l)}\in\RR^m$ and $\bz^{(l)}\in\RR^{d}, y^{(l)}\in \RR$. With these notations, we can re-write the model as
\begin{equation}\label{eqn: forward-prop}
\begin{aligned}
\bz^{(1)} &= \bx, \,\, y^{(1)} = 0 \\
\bz^{(l+1)} &= \bz^{(1)} + C^{(l)} \sigma(B^{(l)}\bz^{(l)} + \br^{(l)} y^{(l)}) \\
y^{(l+1)} & = y^{(l)} + \big(\ba^{(l)}\big)^T \sigma(B^{(l)} \bz^{(l)} + \br^{(l)} y^{(l)}), \quad l=1, \cdots, L-1. \\
f(\bx;\Theta) &= y^{(L)}.
\end{aligned}
\end{equation}

\subsubsection*{Gradient descent}
We will analyze the behavior of the gradient descent algorithm, defined by
\[
    \Theta_{t+1} = \Theta_t - \eta \nabla \hat{\cR}_n(\Theta_t),
\]
where $\eta$ is the learning rate. For simplicity, in most cases, we will focus on its continuous version:
\begin{equation}\label{eqn: GD-deep-net}
    \frac{d\Theta_t}{dt} = - \nabla \hat{\cR}_n(\Theta_t).
\end{equation}

\paragraph*{Initialization} We will focus on a special class of initialization:
\begin{equation}
\Cl_0 = 0,\,\, \al_0=0,\quad    \sqrt{m} \text{row}(\Bl_0)\sim \pi_0,\,\,\rl_0=0,
\label{eqn:init}
\end{equation}
where the third item means that each row of $\Bl_0$ is independently drawn from the uniform distribution over  $\{\bb\in\RR^d : \|\bb\|=1/\sqrt{m}\}$. Thus for this initialization, $\|\Bl_0\|_F = 1$.

Note that all the results in this paper also hold for  slightly larger initializations, e.g. $\max_{l\in [L]}\{\|\Cl_0\|_F,\|\al_0\|,\|\rl_0\|\}=O(1/L)$ and $\text{row}_i(\Bl_0) \sim\cN(0,I/m)$. But for simplicity, we  will focus on the initialization~\eqref{eqn:init}. 

\subsection{Assumption on the input data}
\label{sec: data}
For the given  activation function $\sigma$, we can define a symmetric positive definite (SPD) function
\begin{equation}
    k_0(\bx,\bx') \Def \EE_{\bw\sim\pi_0}[\sigma(\bw^T\bx)\sigma(\bw^T\bx')].
\end{equation}
Denote by $\cH_{k_0}$  the RKHS induced by $k_0(\cdot,\cdot)$. For the given training set, the (empirical) kernel matrix $K = (K_{i,j}) \in\RR^{n\times n}$ is defined as
\[
    K_{i,j}  =  \frac{1}{n}k_0(\bx_i,\bx_j), \quad i,j=1, \cdots, n
\]
We make the following assumption on the training data set.
\begin{assumption}\label{assump: data}
For the given training data $\{\bx_i\}_{i=1}^n$, we assume that  $K$ is positive definite, i.e.
\begin{equation}
    \lambda_n \Def \lambda_{\min}(K)>0.
\end{equation}
\end{assumption}


\begin{remark}
Note that $\lambda_n \leq \min_{i\in [n]} K_{i,i}\leq 1/n$, and in general  $\lambda_n$ depends on the data set.
 If we assume that $\{\bx_i\}_{i=1}^n$ are independently drawn from $\pi_0$, it was shown in  \cite{braun2006accurate} that with high probability $\lambda_n\geq \mu_n/2$ where $\mu_n$ is the $n$-th eigenvalue of the Hilbert-Schmidt integral operator $T_{k_0}: L^2(\SS^{d-1},\pi_0)\mapsto L^2(\SS^{d-1},\pi_0)$ defined by
\[
    T_{k_0} f(\bx) = \int_{S^{d-1}} k_0(\bx,\bx')f(\bx')d\pi_0(\bx').
\]
 Using this result, \cite{xie2016diverse} provided lower bounds for $\lambda_n$ based on some  geometric discrepancy.  

\end{remark}

\section{The main results}
Let $\Theta_t$ be the solution of the GD dynamics \eqref{eqn: GD-deep-net} at time $t$ with the initialization defined
in \eqref{eqn:init}.
We first show that with  high probability, the landscape of $\hat{\cR}_n(\Theta)$ near the initialization  has some
coercive property which guarantees the exponential convergence towards a global minimum.
\begin{lemma}\label{pro: GD-general}
Assume that there are constants $c_1, c_2, c_3$ such that
 $4c_3J(\Theta_0)<c_1^2c_2^2$ and  for  $\|\Theta-\Theta_0\|\leq c_1$
 \begin{equation}\label{eqn: grad-cond}
     c_2 J(\Theta) \leq \|\nabla J(\Theta)\|^2 \leq c_3 J(\Theta).
 \end{equation}
Then   for any $t\geq 0$,  we have
 \[
     J(\Theta_t)\leq e^{-c_2 t}J(\Theta_0).
 \]
\end{lemma}
\begin{proof}
Let ${ t_0 \Def \inf \{t : \|\Theta_t-\Theta_0\|\geq c_1\} }$.
Then for $t\in[0,t_0]$, the condition \eqref{eqn: grad-cond} is satisfied. Thus we have
$$
    dJ/dt = - \|\nabla J\|^2\leq -c_2 J.
$$
Consequently, we have,
$$J(\Theta_t)\leq e^{-c_2t} J(\Theta_0)$$
  It  remains to  show that actually $t_0=\infty$. If $t_0<\infty$, then we have
\begin{align*}
    \|\Theta_{t_0}-\Theta_0\|& \leq \int_0^{t_0} \|\nabla J(\Theta_t) \| dt \leq \sqrt{c_3 J (\Theta_0)} \int_0^{\infty} e^{-c_2t/2}dt \leq \frac{2\sqrt{c_3 J(\Theta_0)}}{c_2} \stackrel{(i)}{<} c_1,
\end{align*}
where $(i)$ is due to the assumption  that $4c_3 J(\Theta_0)<c_1^2c_2^2$. This contradicts the definition of $t_0$. 
\end{proof}


Our main result for optimization is as follows.
\begin{theorem}[Optimization]\label{thm:continues_gd}
For any $\delta \in (0, 1)$, assume that $L\gtrsim \max\{\lambda_n^{-2} \ln(n^2/\delta),\lambda_n^{-3}\}$.
With probability at least $1-\delta$ over the initialization $\Theta_0$, we have that for any $t\geq 0$,
\begin{equation}
    \hat{\cR}_n(\Theta_t)\leq e^{-\frac{L \lambda_n t}{2}}\hat{\cR}_n(\Theta_0).
\end{equation}
\end{theorem}
In contrast to other recent results for multi-layer neural networks~\cite{du2018gradient,allen2018convergence,zou2018stochastic},
we do not require the network width to increase with the size of the data set or the depth of the network.


As is the case for two-layer neural networks \cite{e2019twolayer}, the fact that the GD dynamics stays in a neighborhood of the initialization suggests that it resembles the
situation of a random feature model.
Consequently, the generalization error can be controlled if we assume that the target function is in the appropriate RKHS.
\begin{assumption}\label{assump: target-function}
Assume that $f^*\in \cH_{k_0}$, i.e.
\begin{align*}
    f^*(\bx) &= \EE_{\pi_0}[a^*(\omega) \sigma(\omega^T\bx)]\\
    \|f^*\|^2_{\cH_{k_0}} & = \EE_{\pi_0}[|a^*(\omega)|^2] < +\infty.
\end{align*}
In addition, we also assume that $\sup_{\omega\in \SS^{d-1}}|a^*(\omega)|<\infty$.
\end{assumption}
In the following, we will denote $\gamma(f^*)=\max\{1,\sup_{\omega\in\SS^{d-1}}|a^*(\omega)|\}$.
Obviously, $\|f^*\|_{\cH_{k_0}} \leq \gamma(f^*)$. 

\begin{theorem}[Generalization]\label{thm: generalization}
Assume that the target function $f^*$ satisfies Assumption \ref{assump: target-function}. For any $\delta \in (0, 1)$, assume that $L\gtrsim \max\{\lambda_n^{-2}\ln(n^2/\delta),\lambda_n^{-3}, \gamma^2(f^*)\}$. Then with probability at least $1-\delta$ over the random initialization,  the following  holds for any $t\in [0,1]$,
\[
\cR(\Theta_t)\lesssim \frac{1}{L^{3/2}\lambda^2_n}+ \frac{t}{\lambda_n^2}  + \frac{\gamma^2(f^*)}{Lt} + \left(1 + \frac{\sqrt{L}\gamma(f^*)t}{\sqrt{n}}\right)^2 \frac{c^3(\delta)\gamma^2(f^*)}{\sqrt{n}}.
\]
where $c(\delta)=1+\sqrt{2\ln(1/\delta)}$.
\end{theorem}
In addition, by choosing the stopping time appropriately, we obtain the following result:
\begin{corollary}[Early-stopping]\label{col: gen-err}
Assume that $L\gtrsim \max\{\lambda_n^{-2} \ln(n^2/\delta),\lambda_n^{-3}, \gamma^2(f^*)\}$.
Let  $T= \sqrt{n}/L$, then we have
\[
\cR(\Theta_T)\lesssim c^3(\delta) \frac{\gamma^2(f^*)}{\sqrt{n}}.
\]
\end{corollary}


\section{Landscape around the initialization}%
\label{sec:landscape}
\begin{definition}
For any $c>0$ , we define a neighborhood around the initialization $\Theta_0$ by
\begin{equation}
    \cI_c(\Theta_0) \Def \{\Theta : \max_{l\in [L]}\{\|\ba^{(l)}-\al_0\|,\|\br^{l}-\rl_0\|, \|B^{(l)}-\Bl_0\|_{F},\|\Cl-\Cl_0\|_{F}\}\leq \frac{c}{L}\}.
    \label{eqn:neighbor}
\end{equation}
\end{definition}
Let $\varepsilon_c = c/L$. We will assume  that $c\geq 1, \varepsilon_c\leq 1$.
In the following, we first prove that  both the forward and backward propagation is stable regardless of its depth.
We then show that the norm of the gradient can be bounded from above and below by the loss function, similar to the condition
required in Lemma~\ref{pro: GD-general}. This implies that  there are no issues with   vanishing or exploding
gradients.

\subsection{Forward stability}
At $\Theta_0$, it easy to check that
\[
    \yl(\bx; \Theta_0)=0,\quad \zl(\bx; \Theta_0)=\bx,\quad \gl(\bx;\Theta_0)=\sigma(\Bl_0\bx).
\]
For simplicity, when it is clear from the context, we will omit the dependence on $\bx$ and $\Theta$ in the notations.

\begin{proposition}\label{pro: forward-stability}
If $L\geq 4c^2$, we have for any $\Theta\in \cI_c(\Theta_0)$ and $\bx\in\SS^{d-1}$ that
\begin{equation}
\begin{aligned}
    |\yl(\bx;\Theta)-\yl(\bx;\Theta_0)| &\leq 4c\\
     \|\zl(\bx;\Theta)-\zl(\bx;\Theta_0)\| &\leq \frac{4c}{L}\\
    \|\gl(\bx;\Theta)-\gl(\bx;\Theta_0)\| & \leq \frac{6c^2}{L}.
\end{aligned}
\end{equation}
\end{proposition}

\begin{remark}
We  see that all the variables are close to their initial value except $\yl$, which is
 used to accumulate the prediction from  each layer.
\end{remark}

\begin{proof}
Let $\ol=\|\zl(\bx;\Theta)-\bz^{(l)}(\bx;\Theta_0)\|$. Then by~\eqref{eqn: forward-prop}, we have
\begin{align*}
o^{(l+1)} &\leq \varepsilon_c \left((1+\varepsilon_c)(1+\ol)+\varepsilon_c |\yl|\right) \\
|y^{(l+1)}| &\leq |\yl| + \varepsilon_c \left((1+\varepsilon_c)(1+\ol)+\varepsilon_c|\yl|\right),
\end{align*}
with $o^{(1)}=0,y^{(1)}=0$.
Adding the two inequalities gives us:
\begin{align*}
o^{(l+1)} +  |y^{(l+1)}| &\leq 2\varepsilon_c(1+\varepsilon_c) o^{(l)} + (1+2\varepsilon_c^2) |y^{(l)}| + 2 \varepsilon_c(1+\varepsilon_c)
\end{align*}
Since $2\varepsilon_c= 2c/L\leq 1$, the above inequality can be simplified as
\begin{align*}
o^{(l+1)} + |y^{(l+1)}|& \leq (1+2\varepsilon_c^2)(o^{(0)} + |y^{(0)}|) + 2.25\varepsilon_c\\
& \leq 2.25\varepsilon_c\sum_{l'=0}^l (1+2\varepsilon_c)^{l'}\leq 2.25L\varepsilon_ce^{L 2\varepsilon_c^2}\leq 4c
\end{align*}
Thus we obtain that for any $l\in [L]$, $|\yl| \leq  4c$.
Plugging it back to the recursive formula for $\ol$, we get
\begin{equation*}
o^{(l+1)}\leq 1.25 \varepsilon_c \ol + 2.25\varepsilon_c
\end{equation*}
This gives us
\begin{equation*}
 \ol \leq 4c/L\quad  \forall \,\,l\in [L].
\end{equation*}
Now the deviation of $\gl$ can be estimated by
\begin{align*}
 \|\gl(\bx;\Theta)-\gl(\bx;\Theta_0)\| &= \|\sigma(\Bl\zl(\bx) + \rl\yl(\bx)) - \sigma(\Bl_0\bx)\|\\
&\leq \|\Bl\zl(\bx) + \rl\yl(\bx)-\Bl_0\bx\|
\end{align*}
By inserting the previous estimates, we  obtain
\[
 \|\gl(\bx;\Theta)-\gl(\bx;\Theta_0)\|\leq \frac{6c^2}{L}
\]

\end{proof}

\subsection{Backward stability}
For convenience, we define the gradients with respect to the neurons by  
\[
\alpha^{(l)}(\bx;\Theta)=\nabla_{y^{(l)}} f(\bx;\Theta)\quad  \bm{\beta}^{(l)}(\bx;\Theta) = \nabla_{\bz^{(l)}}f(\bx;\Theta)\quad \bm{\gamma}^{(l)}(\bx;\Theta) = \nabla_{\bg^{(l)}}f(\bx;\Theta).
\] 
For simplicity, we will omit the explicit reference of  $\bx$ and $\Theta$ in these notations when it is clear from the context. Note that $\alpha^{(l)}\in\RR, \bm{\gamma}^{(l)}\in\RR^m, \bm{\beta}^{(l)}\in\RR^d$, and
it is easy to derive the following back-propagation formula  using the chain rule,
\begin{equation}\label{eqn: back-prop}
\begin{aligned}
\bm{\gamma}^{(l)} &= \ba^{(l)} \alpha^{(l+1)} + \big(C^{(l)}\big)^T \bm{\beta}^{(l+1)}\\
\bm{\beta}^{(l)} &= \big(B^{(l)}\big)^T \bm{\gamma}^{(l)}\\
\alpha^{(l)} & = \alpha^{(l+1)} + (\br^{(l)})^T \bm{\gamma}^{(l)}.
\end{aligned}
\end{equation}
At the top layer, we have that for any $\Theta$ and $\bx$:
\[
    \alpha^{(L)} = 1,\quad \bm{\beta}^{(L)} = 0.
\]
In addition, we have at $\Theta_0$
\[
\alphal(\bx;\Theta_0)=1, \qquad \betal(\bx;\Theta_0)=0, \qquad \gammal(\bx;\Theta_0)=0.
\]
\begin{proposition}\label{pro: backward-stability}
If $L\geq 6c^2$, we have for any $\Theta\in\cI_c(\Theta_0)$ and $\bx$
\begin{equation}
|\alpha^{(l)}(\bx;\Theta)-1|\leq \frac{5c}{L}, \quad \|\betal(\bx;\Theta)\|\leq \frac{4c}{L},\quad \|\gammal(\bx;\Theta)\|\leq \frac{3c}{L}
\end{equation}
\end{proposition}
\begin{proof}
According to the \eqref{eqn: back-prop}, we have
\begin{align*}
    \begin{pmatrix}
        \bm{\beta}^{(l)} \\
        \alpha^{(l)}
    \end{pmatrix} &=
        \begin{pmatrix}
       (C^{(l)}B^{(l)})^T &(B^{(l)})^T\ba^{(l)}   \\
        (C^{(l)}\br^{(l)})^T & 1+(\br^{(l)})^Ta^{(l)}
        \end{pmatrix}
        \begin{pmatrix}
            \bm{\beta}^{(l+1)} \\
            \alpha^{(l+1)}
        \end{pmatrix},
\end{align*}
which gives us 
\begin{align}\label{eqn: xxx}
\|\betal\| &\leq \varepsilon_c(1+\varepsilon_c) \|\bm{\beta}^{(l+1)}\| + \varepsilon_c(1+\varepsilon_c) \alpha^{(l+1)} \\
|\alphal| & \leq \varepsilon_c^2 \|\bm{\beta}^{(l+1)}\| + (1+\varepsilon_c^2) \alpha^{(l+1)}.
\end{align}
Adding them, we obtain
\[
   \varepsilon_c\|\betal\| + \alphal \leq \varepsilon_c^2(2+\varepsilon_c) \|\bm{\beta}^{(l+1)}\| + (1+2.25\varepsilon_c^2) \alpha^{(l+1)}\leq (1+2.25\varepsilon_c^2) ( \|\bm{\beta}^{(l+1)}\| +  \alpha^{(l+1)}).
\]
Therefore, we have
\[\alphal \leq  \varepsilon\betal + \alphal \leq (1+2.25 \varepsilon_c^2)^L\leq 1+5c\varepsilon_c.
\]
Inserting the above estimates back to~\eqref{eqn: xxx} gives us
\[
    \|\betal\|\leq 1.25\varepsilon_c \|\bm{\beta}^{(l+1)}\| + 2.5\varepsilon_c,
\]
from which we obtain that
$$\|\betal\|\leq 4\varepsilon_c$$.
Using the \eqref{eqn: back-prop} again, we get
\[
    \|\bm{\gamma}^{(l)}\|= \|\ba^{(l)} \alpha^{(l+1)} + \big(C^{(l)}\big)^T \bm{\beta}^{(l+1)}\|\leq 3\varepsilon_c.
\]
For the lower bound, using \eqref{eqn: back-prop}, we get
\begin{align*}
\alpha^{(l)} & = \alpha^{(l+1)} + (\br^{(l)})^T \bm{\gamma}^{(l)}\geq \alpha^{(l+1)} - 3\varepsilon_c^2\\
&\geq \alpha^L - 3L\varepsilon_c^2\geq 1 - \frac{3c^2}{L}.
\end{align*}

\end{proof}

\subsection{Bounding the gradients}
\label{sec: gradient-bound}
We are now ready to bound the gradients. First note that we have
\begin{align*}
    \nabla_{\al}f(\bx) &= \alphal(\bx) \gl(\bx) \\
    \nabla_{\Bl}f(\bx) &= \gammal(\bx) \big(\zl(\bx)\big)^T \\
    \nabla_{\Cl}f(\bx) &= \betal(\bx) \big(\gl(\bx)\big)^T \\
    \nabla_{\rl}f(\bx) &= \gammal(\bx) \yl(\bx),
\end{align*}
where we have omitted the dependence on $\Theta$. Using the  stability results, we can bound the gradients by the empirical loss.
\begin{lemma}[Upper bound]\label{lemma: upper-bound}
If $L\geq 100c^2$, then for any  $\Theta\in\cI_c(\Theta_0)$  we have
\begin{equation}
\begin{aligned}
    \max\{\,\|\nabla_{\al}\hat{\cR}_n\|^2,\|\nabla_{\rl}\hat{\cR}_n\|^2\,\} &\leq(1+ \frac{50c^2}{L})\hat{\cR}_n \\
    \max\{\,\|\nabla_{\Bl}\hat{\cR}_n\|^2,\|\nabla_{\Cl}\hat{\cR}_n\|^2 \,\} &\leq  \frac{20c^2}{L^2}\hat{\cR}_n
\end{aligned}
\end{equation}
\end{lemma}
\begin{proof}
Using Lemma~\ref{pro: forward-stability} and Lemma~\ref{pro: backward-stability}, we have
\begin{align*}
\|\nabla_{\al}\hat{\cR}_n\|^2 &= \|\frac{1}{n}\sum_{i=0}^n e(\bx_i,y_i) \alpha^{(l)}(\bx_i) \gl(\bx_i)\|^2\\
&\leq \hat{\cR}_n(\Theta)\frac{1}{n}\sum_{i=1}^n \|\alpha^{(l)}(\bx_i) \gl(\bx_i)\|^2\\
&\leq \frac{\hat{\cR}_n(\Theta)}{n}\sum_{i=0}^n (1+\frac{5c}{L})^2(1+\frac{6c^2}{L})^2 \\
&\leq (1+\frac{50c^2}{L}) \hat{\cR}_n(\Theta)
\end{align*}
Analogously, we have
\begin{align*}
(A).&\qquad \|\nabla_{\Bl}\hat{\cR}_n\|_F^2 = \|\frac{1}{n}\sum_{i=0}^n e(\bx_i,y_i) \gammal(\bx_i) \big(\zl(\bx_i)\big)^T\|_F^2 \\
&\qquad\qquad\qquad\quad\leq \hat{\cR}_n(\Theta) \frac{1}{n}\sum_{i=1}^n (\frac{3c}{L})^2 (1+\frac{4c}{L})^2\lesssim \frac{15c^2}{L^2} \hat{\cR}_n(\Theta);\\
(B). &\qquad \|\nabla_{\Cl}\hat{\cR}_n\|_F^2 = \|\frac{1}{n}\sum_{i=0}^n e_i \betal(\bx_i) \gl_k(\bx_i)\|_F\\
&\qquad\qquad\qquad\quad \leq\hat{\cR}_n(\Theta)\frac{1}{n}\sum_{i=0}^n (\frac{4c}{L})^2(1+\frac{6c^2}{L})^2 \leq\frac{20c^2}{L^2}\hat{\cR}_n(\Theta); \\
(C). &\qquad \|\nabla_{\rl}\hat{\cR}_n\|^2 = \|\frac{1}{n}\sum_{i=0}^n e_i \yl(\bx_i)\gammal(\bx_i)\|^2 \\
&\qquad\qquad\qquad\quad \leq \hat{\cR}_n(\Theta)\frac{1}{n}\sum_{i=0}^n (4c)^2 (\frac{3c}{L})^2 \leq \frac{1}{2} \hat{\cR}_n(\Theta).
\end{align*}

\end{proof}

We now turn to the lower bound. The technique used is similar to case for  two-layer
neural networks~\cite{du2018gradient}. Define a Gram matrix $H = (H_{i,j}) \in\RR^{n\times n}$ with
\begin{equation}
    H_{i,j}(\Theta) = \frac{1}{nL}\sum_{l=1}^L\langle\nabla_{\al} f(\bx_i) , \nabla_{\al}f(\bx_j)\rangle.
    \label{eqn:gram}
\end{equation}
At the initialization, we have
\[
    H_{i,j}(\Theta_0) = \frac{1}{nL}\sum_{l=1}^L\langle \sigma(\Bl_0\bx_i),\sigma(\Bl_0\bx_j)\rangle.
\]
This matrix can be viewed as an empirical approximation of the kernel matrix $K$ defined in Section~\ref{sec: data},
{ since  each row of $\Bl_0$ is independently drawn from the uniform distribution over the sphere of radius $1/\sqrt{m}$ }.
Using standard concentration inequalities, we can prove that with  high probability, the smallest eigenvalue of the Gram matrix is  bounded from below by the smallest eigenvalue of the kernel matrix.
This is stated in the following lemma, whose proof is deferred to Appendix~\ref{sec: init-Gram-matrix}.
\begin{lemma}\label{lem:lambda0}
For any $\delta\in (0,1)$, assume that $L\geq \frac{8\ln(n^2/\delta)}{m\lambda_n^2}$. Then
 with probability at least $1-\delta$ over the random initialization:
\begin{equation}
\lambda_{\min}(H(\Theta_0))\geq \frac{3\lambda_n}{4}.
\end{equation}
\end{lemma}

Moreover, we can show that for any  $\Theta\in\cI_c(\Theta_0)$, the Gram matrix $H(\Theta)$ is still strictly positive definite as long as $L$ is large enough.
\begin{lemma}\label{lemma: gram-matrix-minimum-eigval}
For any $\delta \in (0,1)$, assume that $L\geq \max\{\frac{8\ln(n^2/\delta)}{m\lambda_n^2},\frac{200c^2}{\lambda_n}\}$.   With probability $1-\delta$ over the random initialization, we have for any $\Theta\in\cI_c(\Theta_0)$,
\begin{equation}
    \lambda_{\min}(H(\Theta))\geq \frac{\lambda_n}{2}.
\end{equation}
\end{lemma}
\begin{proof}
\begin{align}\label{eqn: gram-perturb}
\nonumber
H_{i,j}(\Theta) - H_{i,j}(\Theta_0) &= \\
 \frac{1}{nL}\sum_{l=1}^L \langle\alphal(\bx_i;\Theta)&\gl(\bx_i;\Theta),\alphal(\bx_j;\Theta)\gl(\bx_j;\Theta)  \rangle   - \langle\gl(\bx_i;\Theta_0),\gl(\bx_j;\Theta_0)
\end{align}
Lemmas~\ref{pro: forward-stability} and \ref{pro: backward-stability} tell us that for any $\bx_i\in \SS^{d-1}$
\[
    \|\gl(\bx_i;\Theta)-\gl(\bx_i;\Theta_0)\|\leq \frac{6c^2}{L},\qquad |\al(\bx_i;\Theta)-1|\leq\frac{5c}{L}.
\]
Using  these in \eqref{eqn: gram-perturb} gives us
\[
    |H_{i,j}(\Theta)-H_{i,j}(\Theta_0)|\leq \frac{50c^2}{nL}.
\]
{ Applying  Weyl's inequality that $\sigma_{\min}(A_1+A_2)\geq \sigma_{\min}(A) - \sigma_{\max}(A_2)$, }we  obtain
\begin{align*}
    \lambda_{\min}(H(\Theta))& \geq \lambda_{\min}(H(\Theta_0))-\|H(\Theta)-H(\Theta_0)\|_2\\
    &\geq \frac{3}{4}\lambda_n - \frac{50c^2}{L}.
\end{align*}
Thus as long as $L\geq 200c^2\lambda_n^{-1}$, we must have $\lambda_{\min}(H(\Theta))\geq \lambda_n/2$.

\end{proof}
With the above lemma, we can now provide a lower bound for the square norm of the gradient.
\begin{lemma}[Lower bound]\label{lemma: lower-bound}
For any fixed $\delta\in (0,1)$, assume that $L\geq \max\{\frac{8\ln(n^2/\delta)}{m\lambda_n^2},\frac{200c^2}{\lambda_n}\}$.  With probability at least $1-\delta$ over the random initialization, we have for any $\Theta\in \cI_c(\Theta_0)$,
 the empirical risk satisfies
\begin{equation}
    \|\nabla\hat{\cR}_n(\Theta)\|^2 \geq \frac{\lambda_nL}{2} \hat{\cR}_n(\Theta).
\end{equation}
\end{lemma}
\begin{proof}
\begin{align*}
\|\nabla\hat{\cR}_n(\Theta)\|^2 &\geq \sum_{l=1}^L \|\nabla_{\al}\hat{\cR}_n(\Theta)\|^2 = \sum_{l=1}^L \|\frac{1}{n}\sum_{i=1}^n e(\bx_i,y_i) \nabla_{\al} f(\bx_i)\|^2 \\
&=\frac{L}{n}\be^T H(\Theta)\be \geq L \lambda_{\min}(H(\Theta))\hat{\cR}_n(\Theta).
\end{align*}
Applying Lemma~\ref{lemma: gram-matrix-minimum-eigval} completes the proof.
\end{proof}

\section{Optimization}

\paragraph*{Proof of Theorem~\ref{thm:continues_gd}}
Let $c_n = 8/\lambda_n$ and define a stopping time
{
\[
    t_0 \Def \inf \,\{\,t : \Theta_t\notin \cI_{c_n}(\Theta_0)\}.
\]
}
By Lemma~\ref{lemma: gram-matrix-minimum-eigval}, since
\[
    L\geq \max\{\frac{8\ln(n^2/\delta)}{m\lambda_n^2},\frac{200c_n^2}{\lambda_n}\},
\]
we have that with probability at least $1-\delta$ over the choice of $\Theta_0$, the following inequality holds for any $0\leq t\leq t_0$:
\[
\|\nabla\hat{\cR}_n(\Theta_t)\|^2 \geq \frac{\lambda_nL}{2} \hat{\cR}_n(\Theta_t).
\]
Hence we have
\[
    \hat{\cR}_n(\Theta_t)\leq e^{-\frac{\lambda_n Lt}{2}}\hat{\cR}_n(\Theta_0), \quad\forall\; t\leq t_0.
\]
Now we prove that $t_0=\infty$. If $t_0< \infty$, we
must have $\Theta_{t_0} \in \partial(\cI_{c_n}(\Theta_0))$,  the boundary of $\cI_{c_n}(\Theta_0)$. It is easy to see that for any $l\in [L]$,
\begin{align}
    \|\al_{t_0} -\al_0\|
    & \leq \int_0^{t_0}\|\nabla_{\al}\hat{\cR}_n\|dt
    \stackrel{(i)}{\leq} \int_0^{\infty} \sqrt{(1+50c_n^2/L)\hat{\cR}_n(\Theta_t)}dt
    \nonumber \\
    &\leq \sqrt{(1+50c_n^2/L)} \int_0^{\infty}
    e^{-\frac{L\lambda_n t}{4}}\sqrt{\hat{\cR}_n(\Theta_0)}dta \nonumber \\
    &\leq \frac{4\sqrt{2}}{L\lambda_n} <\frac{c_n}{L}
\end{align}
where $(i)$ is due to Lemma~\ref{lemma: upper-bound}. Similarly we have  $\|\rl_{t_0} -\rl_0\|<c_n/L$, and
\begin{equation}
\begin{aligned}
\max\{\|\Cl_{t_0}-\Cl_0\|,\|\Bl_{t_0}-\Bl_0\|\}\leq \sqrt{20c^2_n/L^2}\frac{4}{L\lambda_n}<\frac{c_n}{L}
\end{aligned}
\end{equation}
This says that $\Theta_{t_0}\in \cI_{c_n}^{\circ}(\Theta_0)$, which contradicts the definition of $t_0$.

The above proof also suggests that with high probability, $\Theta_t$ is always close to the initialization.
\begin{proposition}\label{pro: close}
For any $\delta\in (0,1)$, assume that $L\geq \max\{8\lambda_n^{-2}\ln(n^2/\delta),3000\lambda_n^{-3}\}$. With probability at least $1-\delta$ over the initialization $\Theta_0$, we have that for any $t\geq 0$,
\[
\Theta_t \in \cI_{c_n}(\Theta_0),
\]
where $c_n=8/\lambda_n$.
\end{proposition}

\section{Generalization}
\subsection{The reference model}
To analyze the generalization error, we will consider  the following random feature model~\cite{rahimi2008random} as a reference model,
\[
    \tilde{f}(\bx; \ba,B_0) \Def \ba^T\sigma(B_0\bx)=\sum_{l=1}^L (\al)^T \sigma(\Bl_0\bx),
\]
where $\ba^{(l)}\in\RR^{m}, \Bl_0 \in\RR^{m\times d}$ and $\ba\in\RR^{mL}, B_0\in\RR^{mL\times d}$ denote the stacked parameters. The $\sigma(B_0\bx)$ and $\ba$ are the random features and  coefficients, respectively. The initialization of $\al_0$ and $\Bl_0$ are the same as the deep neural networks.  $\{\Bl_0\}_{l=1}^L$ are kept fixed during the training after the initialization, while $\{\al\}_{l=1}^L$ are updated according to gradient descent. For this model, we define the empirical and population risks by
\[
    \hat{\cE}_n(\ba; B_0) \Def \frac{1}{2n}\sum_{i=1}^n (\tilde{f}(\bx_i;\ba,B_0)-y_i)^2,
    \qquad
    \cE(\ba; B_0) \Def \half\EE_{\bx,y}[(\tilde{f}(\bx;\ba,B_0)-y)^2].
\]

Concerning the reference model, we have
\begin{theorem}\label{thm: approx}
Assume that the target function $f^*$ satisfies Assumption~\ref{assump: target-function}.  Then for any fixed $\delta \in (0,1)$,  with probability at least $1-\delta$ over the random choice of $B_0$,
there exists $\ba^*\in\RR^{mL}$ such that
{
\[
    \cE(\ba^*;B_0)\leq c^2(\delta)\frac{\gamma^2(f^*)}{mL},
\]
 where $c(\delta)=1+\sqrt{2\ln(1/\delta)}$.}
Furthermore, $\|\ba^*\|\leq \gamma(f^*)/\sqrt{L}$.
\end{theorem}

This result essentially appeared in \cite{rahimi2009weighted,rahimi2008uniform}.
Since we are interested in the explicit control for the norm of the solution, we provide a complete proof  in  Appendix~\ref{sec: random-feature}.

\paragraph*{Gradient descent for the random feature model}
Denote by $\ta_t$ the solution of GD for the random feature model:
\begin{equation}\label{eqn: GD-ref}
\begin{aligned}
\ta_0 &= \ba_0 \\
\frac{d \ta_t}{dt} &= - \nabla\hat{\cE}_n(\ta_t; B_0).
\end{aligned}
\end{equation}
The generalization property  of GD solutions for random feature models was analyzed  in \cite{carratino2018learning}. {Here we provide a much simpler approach based on the following lemma.
\begin{lemma}\label{lem: convergence-reference-dynamics}
For any fixed $B_0$, the gradient descent~\eqref{eqn: GD-ref} satisfies,
\begin{align*}
    \hat{\cE}_n(\tilde{\ba}_t; B_0) & \leq\hat{\cE}_n(\ba^*;B_0)+ \frac{\|\ba_0-\ba^*\|^2}{2t}\\
    \|\ta_t-\ba^*\| & \leq  \|\ba^0-\ba^*\|+ 2t\hat{\cE}_n(\ba^*;B_0).
\end{align*}
\end{lemma}
The proof can be found in Appendix~\ref{sec: rdf-2}. Here the key observation is  the second inequality. In the case $\ba_0=0$, we have
\begin{align*}\label{eqn: convex-dist}
    \|\ta_t\|\leq 2\|\ba^*\| + 2t\hat{\cE}_n(\ba^*;B_0).
\end{align*}
Note that Theorem~\eqref{thm: approx} implies  that $\hat{\cE}_n(\ba^*;B_0)$ is small. This gives us a control over the norm of GD solutions.
We can now  obtain the following estimates of population risk.
\begin{theorem}\label{thm: generalization-random-feature}
Assume that the target function $f^*$ satisfies Assumption~\ref{assump: target-function} and $L\gtrsim \gamma^2(f^*)$. Then for any fixed $\delta\in (0,1)$, with probability at least $1-\delta$ over the random choice of $B_0$, the following  holds for any $t\in [0,1]$
\begin{align*}
\cE(\ta_t;B_0)&\lesssim \frac{c^2(\delta)\gamma^2(f^*)}{mL} + \frac{\gamma^2(f^*)}{2Lt} + \left(1 + \frac{\sqrt{L}\gamma(f^*)t}{\sqrt{n}}\right)^2 \frac{c^3(\delta)\gamma^2(f^*)}{\sqrt{n}}.
\end{align*}
\end{theorem}
 The three terms at the right hand side are bounds for the approximation error, optimization error and estimation error, respectively.
 The proof of this result can be found in Appendix~\ref{sec: rdf-3}.

}

\subsection{Bounding the difference between the two models}
We use $\theta$ to represent $C,\br$ and denote the deep net by $f(\bx; \ba,B,\theta)$. Then for any $\bx$ we have,
\begin{equation}
    f(\bx;\ba,B_0,0) = \tilde{f}(\bx; \ba,B_0).
\end{equation}
In particular,  when $\Theta$ is close to the initialization, we have the following bounds for the two models.
\begin{lemma}\label{lem: predict-diff}
If $\Theta = (\ba,B,\theta)\in \cI_{c}(\Theta_0)$, then we have { for any $\bx \in \SS^{d-1}$}
\begin{align*}
    |f(\bx;\ba,B,\theta)-\tilde{f}(\bx;\ba,B_0)| & \leq 6 c^2 \|\ba\|/L \\
    \|\nabla_{\al} f(\bx;\ba,B,\theta) - \nabla_{\al} \tilde{f}(\bx; \ba, B_0) & \|\leq \frac{c^2}{L}.
\end{align*}
\end{lemma}
\begin{proof}
Define $\tilde{f}^{(l)}(\bx)=\sum_{i=1}^{l}\al \sigma(\Bl_0\bx)$, then we have
\[
   \tilde{f}^{(l+1)}(\bx) - y^{(l+1)}(\bx)= \tilde{f}^{(l)}l(\bx) - \yl(\bx) + \al (\gl(\bx) - \sigma(\Bl_0\bx))
\]
Proposition~\ref{pro: forward-stability} implies that
\[
    |\al (\gl(\bx) - \sigma(\Bl_0\bx))|\leq 6\|\al\|c^2/L.
\]
Hence we have
\[
    |f(\bx;\ba,B,\theta)-f(\bx;\ba,B_0,0)|\leq 6 c^2 \|\ba\|/L.
\]
By applying  Propositions~\ref{pro: forward-stability} and \ref{pro: backward-stability}, we have, for any $\bx\in\SS^{d-1}$
\begin{align*}
    \|\nabla_{\al} f(\bx;\ba,B,\theta) - \nabla_{\al} \tilde{f}(\bx; \ba, B_0)\| & = \|\alphal(\bx)\gl(\bx) - \sigma(\Bl_0\bx)\| \\
    &\leq \|\alphal(\bx)(\gl(\bx)-\gl_0(\bx))\| + \|(\alphal(\bx)-1)\gl_0(\bx)\| \\
    &\leq (1+\frac{5c}{L})\frac{6c^2}{L} + \frac{5c}{L} \\
    &\lesssim \frac{c^2}{L}.
\end{align*}
\end{proof}

Now we can proceed to bound the deviation between the GD dynamics of the two models. Let $\Theta_t = (\ba_t, B_t, \theta_t)$ denote the solution of gradient descent for the deep neural network model.  We have
\begin{lemma}\label{lem: GD-diff}
For any $\delta\in (0,1)$, assume that $L\gtrsim \max\{\lambda_n^{-2}\ln^2(n/\delta),\lambda_n^{-3}\}$. With probability at least $1-\delta$ over the random initialization, we have  for $t \in [0, 1]$
\[
\|\ba_t - \ta_t\|\lesssim \frac{c_n^2t}{\sqrt{L}}, \quad
\]
\end{lemma}
\begin{proof}
First we have that
\begin{align*}
\frac{d(\ba_t-\ta_t)}{dt} &= \frac{-1}{n}\sum_{i=1}^n (f(\bx_i;\Theta_t)-y_i)\nabla_{\ba}f(\bx_i;\Theta_t) - (\tilde{f}(\bx_i;\ta_t,B_0)-y_i)\nabla_{\ba}\tilde{f}(\bx_i;\ta_t,B_0) \\
&= - \frac{1}{n}\sum_{i=1}^{n} (f(\bx_i;\Theta_t)-y_i)\nabla_{\ba}f(\bx_i;\Theta_t)
- (\tilde{f}(\bx_i;\ba_t,B_0)-y_i)\nabla_{\ba}\tilde{f}(\bx_i;\ba_t,B_0) \\
&\quad -\frac{1}{n}\sum_{i=1}^n (\tilde{f}(\bx_i;\ba_t,B_0)-y_i)\nabla_{\ba}\tilde{f}(\bx_i;\ba_t,B_0) - (\tilde{f}(\bx_i;\ta_t,B_0)-y_i)\nabla_{\ba}\tilde{f}(\bx_i;\ta_t,B_0) \\
& =: - \frac{1}{n}\sum_{i=1}^n P_t^i - \frac{1}{n}\sum_{i=1}^n Q_t^i .
\end{align*}
The last equality defines $P_t^i$ and $Q_t^i$ .
Let us estimate $Q^i_t, P^i_t$ separately. We first have 
\begin{align*}
Q_t^i = \sigma(B_0\bx_i)\sigma^T(B_0\bx_i) (\ba_t-\ta_t).
\end{align*}
Proposition~\ref{pro: close} tells us that
$\Theta_t\in \cI_{c_n}(\Theta_0)$ with $c_n=8/\lambda_n$. Hence
\[
    \|\ba_t\|^2 = \|\ba_t-\ba_0\|^2  \leq \sum_{l=1}^L \|\al_t-\al_0\|^2\leq c_n^2/L \leq 1.
\]
Using  Lemma~\ref{lem: predict-diff}, we have
\begin{align}\label{eqn: aaa}
\nonumber \|P_t^i\|&\leq |f(\bx_i;\Theta_t)| \|\nabla_{\ba}f(\bx;\Theta_t)-\nabla_{\ba}\tf(\bx_;\ba_t,B_0)\|\\
\nonumber  &\qquad + |f(\bx_i;\Theta_t)-\tf(\bx_i;\ba_t,B_0)|\|\nabla_{\ba}\tf(\bx_;\ba_t,B_0)\| \\
\nonumber &\qquad + |y_i|\|\nabla_{\ba}f(\bx_i;\Theta_t)-\tf(\bx_;\ba_t,B_0)\| \\
&\leq \frac{2c_n^2 }{\sqrt{L}} + \frac{6c_n^2\|\ba_t\|}{\sqrt{L}}\leq \frac{3c_n^2 }{\sqrt{L}}.
\end{align}
Let $\bm{\delta}_t = \ba_t-\ta_t$,  then
\begin{align*}
    \frac{d\|\bm{\delta}_t\|^2}{dt} &= - \bm{\delta}_t^T\frac{1}{n}\sum_{i=1}^n\sigma(B_0\bx_i)\sigma^T(B_0\bx_i)\bm{\delta}_t - \frac{1}{n}\sum_{i=1}^n \langle \bm{\delta}_t, P_t^i\rangle, \\
    &\leq  \frac{1}{n}\sum_{i=1}^n \| \bm{\delta}_t\|\| P_t^i\|\leq \frac{3c_n^2}{\sqrt{L}}\|\bm{\delta}_t\|,
\end{align*}
where the last inequality follows from ~\eqref{eqn: aaa}. Thus we have
\[
\frac{d\|\bm{\delta}_t\|}{dt} \leq \frac{3c_n^2}{\sqrt{L}}.
\]
Since $\|\bm{\delta}_0\|=0$, we  have
$$\|\bm{\delta}_t\|\leq\int_0^{t}\frac{d\|\bm{\delta}_{t'}\|}{dt} dt'\leq   \frac{3c_n^2t}{\sqrt{L}}.$$

\end{proof}

\subsection{Implicit regularization}
The previous analysis shows that the gradient descent dynamics of the deep neural network model
stays close that of the reference model. More specifically, we have the following result.
\begin{theorem}\label{thm: uniform-closeness}
For any $\delta \in (0,1)$, assume that $L\gtrsim \max\{\lambda_n^{-2}\ln(n^2/\delta),\lambda_n^{-3}\}$. Then with probability at least $1-\delta$ over the initialization $\Theta_0$, we have that for any $t\geq 0$,
\[
 |f(\bx;\ba_t,B_t,\theta_t)-\tilde{f}(\bx;\ba_t,B_0)|\lesssim \frac{c_n^3}{L^{3/2}}
\]
\end{theorem}
\begin{proof}
By Proposition~\ref{pro: close}, we know that $\Theta_t=(\ba_t,B_t,\theta_t)\in \cI_{c_n}(\Theta_0)$. So we can use Lemma~\ref{lem: predict-diff}, which gives us,
\begin{align*}
|f(\bx;\ba_t,B_t,\theta_t)-\tilde{f}(\bx;\ba_t,B_0)| & \leq \frac{6c_n^2\|\ba_t\|}{L} \\
&\stackrel{(i)}{\leq} \frac{6c_n^3}{L^{3/2}},
\end{align*}
where $(i)$ is due to the fact that we have $\|\ba_t\|^2=\sum_{l=1}^L \|\al\|^2\leq c_n^2/L$ for $\Theta\in \cI_{c_n}(\Theta_0)$.

\end{proof}
The above theorem implies that the functions represented by the GD trajectory are uniformly close to that of the random feature model if $L$ is large enough. This allows us to estimate the population risk for the
deep neural network model using results for the random feature model

\begin{proposition}\label{pro: population-risk-diff}
For any fixed $\delta>0$, assume that $L\gtrsim \max\{\lambda_n^{-2}\ln(n^2/\delta),\lambda_n^{-3}\}$.  Then with probability at least $1-\delta$ over the random initialization  $\Theta_0$,  we have
\[
\cR(\Theta_t)\lesssim \frac{c_n^3}{L^{3/2}} + c_n^2 t + \cE(\ta_t;B_0).
\]
\end{proposition}

\begin{proof}
We write
\begin{equation}
\begin{aligned}
\cR(\ba_t,B_t,\theta_t) &= \cR(\ba_t,B_t,\theta_t) - \cR(\ba_t,B_0,0) + \cR(\ba_t,B_0,0) - \cR(\tilde{\ba}_t,B_0,0) \\
&\qquad + \cR(\ta_t,B_0,0)  \\
&= \cR(\ba_t,B_t,\theta_t) - \cE(\ba_t;B_0) + \cE(\ba_t;B_0) - \cE(\ta_t;B_0) + \cE(\ta_t;B_0).
\end{aligned}
\end{equation}
Let $\Delta_i = \tf(\bx_i;\ba_t,B_0)- f(\bx_i;\Theta_t)$. Since $\Theta_t\in \cI_{c_n}(\Theta_0)$, we  have
\[
|\Delta_i|\leq 6c_n^2 \|\ba_t\|/L\leq 6c_n^3/L^{3/2} \leq 1.
\]
Then
\begin{align*}
|\cR(\ba_t,B_t,\theta_t) - \cE(\ba_t;B_0)| &= |\frac{1}{2n}\sum_{i=1}^n (f(\bx_i;\Theta_t)-y_i)^2 - (f(\bx_i;\Theta_t) - y_i + \Delta_i)^2|\\
&= |\frac{1}{n}\sum_{i=1}^n e_i \Delta_i|  + |\frac{1}{2n}\Delta_i^2| \\
&\leq 2 \hat{\cR}_n(\ba_t,B_t,\theta_t) \max_{i}|\Delta_i| \\
&\leq 12\hat{\cR}_n(\Theta_0) c_n^3/L^{3/2} \leq 6c_n^3/L^{3/2},
\end{align*}
where we used the fact that $\hat{\cR}_n(\Theta_0)\leq 1/2$.
In addition, we have for the reference model,
\[
|\cE(\ba_t;B_0) - \cE(\ta_t;B_0)| \leq \|\ba_t-\ta_t\|\|B_0\| \le 3 c^2_n t.
\]
The proof is completed by combining these estimates.
\end{proof}
{
\paragraph*{Proof of Theorem~\ref{thm: generalization}}
By directly combining Theorem~\ref{thm: generalization-random-feature} and Proposition~\ref{pro: population-risk-diff}, we obtain Theorem ~\ref{thm: generalization}.

\paragraph*{Proof of Corollary~\ref{col: gen-err}}
The condition on $L$ implies that $L\geq \max\{\sqrt{n},\gamma^2(f^*)\}$. Hence $T=\sqrt{n}/L\leq 1$. Using Theorem~\ref{thm: generalization}, we get
\begin{align*}
 \cR(\Theta_T)\leq \frac{1}{L^{3/2}\lambda_n^{2}}+ \frac{\sqrt{n}}{L\lambda_n^{2}} + \frac{\gamma^2(f^*)}{\sqrt{n}} + \left(1 + \frac{\gamma(f^*)}{\sqrt{L}}\right)^2 \frac{c^3(\delta)\gamma^2(f^*)}{\sqrt{n}}.
\end{align*}
Since $L\gtrsim \lambda_n^{-2} n$, the  above inequality can be simplified as
\[
\cR(\Theta_T) \lesssim c^3(\delta)\frac{\gamma^2(f^*)}{\sqrt{n}}.
\]

}

\section{Discrete Time Analysis}
In this section, we will analyze the discrete time gradient descent with a constant
learning rate $\eta$:
\begin{equation}
  \Theta_{t+1} = \Theta_t - \eta \nabla \hat{\cR}_n(\Theta_t), \quad t = 0, 1, 2, \dots
\end{equation}
The following theorem shows
that the linear convergence of empirical risk $\hat{\cR}_n(\Theta_t)$ still holds for the
discrete version.

\begin{theorem}\label{thm:discrete}
Consider the neural network model \eqref{eqn: net-arch} with
  initialization~\eqref{eqn:init}. For any fixed $\delta\in (0,1)$, assume that the depth $L \gtrsim
 \max\{ \lambda_n^{-2}  \ln(n^2 / \delta),\lambda_n^{-3}\}$ and the learning rate $\eta \lesssim
  \frac{\lambda_n}{L}$. Then  with probability at least
  $1-\delta$ over the random initialization, we have
  \begin{equation}
    \hcR(\Theta_t)
    \le {\left( 1 - \frac{\lambda_n \eta L}{8} \right)}^t \hcR(\Theta_0).
  \end{equation}
\end{theorem}

To prove this theorem, we need several preliminary results.

\begin{lemma}\label{lem:discrete_step}
  Assume that $L \gtrsim c^2$, and
  $\Theta_t, \Theta_{t+1} \in \cI_c(\Theta_0)$.
  Then for any $\bx\in \SS^{d-1}$,
  \begin{equation}
    \|\gl_{t+1}(\bx) - \gl_t(\bx)\| \lesssim c \eta \sqrt{\hat{\cR}_n(\Theta_t)}.
  \end{equation}
\end{lemma}

\begin{proof}
  We use $\Delta$ to denote the increment of a vector or matrix from
  iteration $t$ to $t+1$, i.e., $\Delta \bv = \bv_{t+1} - \bv_t$.

  For the weights $\al$, $\Bl$, $\Cl$ and $\rl$, from Lemma~\ref{lemma:
  upper-bound}, we have
  \begin{gather*}
    \|\Delta \al\|^2 = \eta^2 {\|\nabla_{\al} \hat{\cR}_n(\Theta_t)\|}^2
    \lesssim \eta^2 \hcR(\Theta_t), \quad
    \|\Delta \Bl\|_F^2 \gtrsim \frac{c^2 \eta^2}{L^2} \hat{\cR}_n(\Theta_t), \\
    \|\Delta \Cl\|_F^2 \gtrsim \frac{c^2 \eta^2}{L^2} \hcR(\Theta_t), \quad
    \|\Delta \rl\|^2 \gtrsim \frac{c^4 \eta^2}{L^2} \hcR(\Theta_t).
  \end{gather*}

  For the neurons $\zl$, $\yl$ and $\gl$, since $\Theta_t, \Theta_{t+1} \in
  \cI_c(\Theta_0)$, we have
  \begin{align}
    \|\Delta \bz^{(l+1)}\|
    & = \left\| \Cl_{t+1} \gl_{t+1} - \Cl_t \gl_t \right\| \nonumber \\
    & \le \|\Cl_t\| \|\gl_{t+1} - \gl_t\|
    + \|\Cl_{t+1} - \Cl_t\| \|\gl_{t+1}\| \nonumber \\
    & \lesssim \frac{c}{L} \|\Delta \gl\|
    + \frac{c \eta}{L} \sqrt{\hcR(\Theta_t)};
    \label{eqn:delta_al}
  \end{align}
  \begin{align}
    \|\Delta y^{(l+1)}\|
    & = \left\| \Delta \yl + \al_{t+1} \gl_{t+1} - \al_t \gl_t \right\|
    \nonumber \\
    & \le \|\Delta \yl\| + \|\al_t\| \|\gl_{t+1} - \gl_t\|
    + \|\al_{t+1} - \al_t\| \|\gl_{t+1} \| \nonumber \\
    & \lesssim \|\Delta \yl\| + \frac{c}{L} \|\Delta \gl\|
    + \eta \sqrt{\hcR(\Theta_t)};
    \label{eqn:delta_yl}
  \end{align}
  \begin{align}
    \|\Delta \gl\|
    & = \left\| \sigma\left( \Bl_{t+1} \zl_{t+1} + \rl_{t+1} \yl_{t+1} \right)
    - \sigma\left( \Bl_t \zl_t + \rl_t \yl_t \right) \right\| \nonumber \\
    & \le \left\| \Bl_{t+1} \zl_{t+1} + \rl_{t+1} \yl_{t+1}
    - \Bl_t \zl_t - \rl_t \yl_t \right\| \nonumber \\
    & \le \|\Bl_t\| \|\Delta \zl\| + \|\Delta \Bl\| \|\zl_{t+1}\|
    + \|\rl_t\| \|\Delta \yl\| + \|\Delta \rl\| \|\yl_{t+1}\| \nonumber \\
    & \lesssim \|\Delta \zl\| + \frac{c \eta}{L} \sqrt{\hcR(\Theta_t)}
    + \frac{c}{L} \|\Delta \yl\| + \frac{c^3 \eta}{L} \sqrt{\hcR(\Theta_t)}.
    \label{eqn:delta_gl}
  \end{align}
  Plugging~\eqref{eqn:delta_al}~\eqref{eqn:delta_yl} in~\eqref{eqn:delta_gl}, and
  using $\|\Delta \yl\| = \sum_{k=0}^{l-1} \left[ \|\Delta y^{(k+1)}\| - \|\Delta
  y^{(k)}\| \right]$ (since $\Delta y^{(0)} = 0$), we obtain
  \begin{align*}
    \|\Delta \gl\|
    & \lesssim \|\Delta \zl\| + \frac{c}{L} \|\Delta \yl\|
    + \frac{c^3 \eta}{L} \sqrt{\hcR(\Theta_t)} \\
    & \lesssim \left[ \frac{c}{L} \|\Delta \bg^{(l-1)}\|
    + \frac{c \eta}{L} \sqrt{\hcR(\Theta_t)} \right]
    + \frac{c}{L} \sum_{k=0}^{l-1} \left[ \frac{c}{L} \|\Delta \bg^{(k)}\|
    + \eta\sqrt{\hcR(\Theta_t)} \right]
    + \frac{c^3 \eta}{L} \sqrt{\hcR(\Theta_t)} \\
    & \lesssim \frac{c}{L} \|\Delta \bg^{(l-1)}\|
    + \frac{c^2}{L^2} \sum_{k=0}^{l-1} \|\Delta \bg^{(k)}\|
    + \left(\frac{c \eta}{L} + c \eta + \frac{c^3 \eta}{L} \right)
    \sqrt{\hcR(\Theta_t)} \\
    & \lesssim \frac{c}{L} \|\Delta \bg^{(l-1)}\|
    + \frac{c^2}{L^2} \sum_{k=0}^{l-1} \|\Delta \bg^{(k)}\|
    + c \eta \sqrt{\hcR(\Theta_t)}.
  \end{align*}
  Since $L \gtrsim c^2$, by induction, we have $\|\Delta \gl\| \lesssim c
  \eta \sqrt{\hcR(\Theta_t)}$ for $l = 0, 1, \dots, L-1$.
\end{proof}


\begin{lemma}\label{lem:discrete_neighbor2converge}
  Assume that $\lambda' = \lambda_{\min}(H(\Theta_0)) > 0$. Assume further that $L
  \gtrsim c^2  / \lambda'$ and $\eta \lesssim \lambda' / L$.
  If $\Theta_t, \Theta_{t+1} \in \cI_{c}(\Theta_0)$, then
    \begin{equation}
    \hcR(\Theta_{t+1})
    \le \left( 1 - \frac{\lambda' \eta L}{4} \right) \hcR(\Theta_t).
  \end{equation}
\end{lemma}

\begin{proof}
    \begin{multline*}
    \hcR(\Theta_{t+1}) - \hcR(\Theta_t)
    = \frac{1}{2n} \sum_{i=1}^n \left[ {[f(\bx_i; \Theta_{t+1}) - y_i]}^2
    - {[f(\bx_i; \Theta_t) - y_i]}^2 \right] \\
    =\frac{1}{n} \sum_{i=1}^n
    [f(\bx_i; \Theta_t) - y_i] [f(\bx_i; \Theta_{t+1}) - f(\bx_i; \Theta_t)]
    +\frac{1}{2n} \sum_{i=1}^n
    {[f(\bx_i; \Theta_{t+1}) - f(\bx_i; \Theta_t)]}^2.
  \end{multline*}
  Notice that
  \[
    f(\bx; \Theta) = \sum_{l=0}^{L-1} \left\langle \al, \gl(\bx) \right\rangle,
  \]
  we can write
  \[
    \hcR(\Theta_{t+1}) - \hcR(\Theta_t) = I_1 + I_2 + I_3,
  \]
  where
  \begin{equation}\begin{aligned}
    I_1 & = \frac{1}{n} \sum_{i=1}^n [f(\bx_i; \Theta_t) - y_i]
    \sum_{l=0}^{L-1} \left\langle \al_{t+1} - \al_t, \gl_t(\bx_i) \right\rangle,
    \\
    I_2 & = \frac{1}{n} \sum_{i=1}^n [f(\bx_i; \Theta_t) - y_i]
    \sum_{l=0}^{L-1}
    \left\langle \al_{t+1}, \gl_{t+1}(\bx_i) - \gl_t(\bx_i) \right\rangle, \\
    I_3 & =\frac{1}{2n} \sum_{i=1}^n
    {[f(\bx_i; \Theta_{t+1}) - f(\bx_i; \Theta_t)]}^2.
  \end{aligned}\label{eqn:d_risk}\end{equation}
  We will show that $I_1$ is the dominant term that leads to linear convergence,
  and $I_2$, $I_3$ are high order terms.

  For  $I_{1}$, since
  \[
    \al_{t+1} - \al_t = - \eta \nabla_{\al} \hcR(\Theta_t)
    = -\frac{\eta}{2n} \sum_{j=1}^n
    [f(\bx_j; \Theta_t) - y_j] \alphal_t \gl_t(\bx_j),
  \]
  we have
  \[
    I_1 = -\frac{\eta L}{2 n} \sum_{i,j=1}^n
    [f(\bx_i; \Theta_t) - y_i] [f(\bx_j; \Theta_t) - y_j]
    \tilde{H}_{i,j}(\Theta_t),
  \]
  where
  \[
    \tilde{H}_{i,j}(\Theta_t) = \frac{1}{L} \sum_{l=0}^{L-1}
    \left\langle \alphal_t \gl_t(\bx_j), \gl_t(\bx_i) \right\rangle.
  \]
  Following the proof of Lemma~\ref{lemma: gram-matrix-minimum-eigval}, for
  $\Theta_t \in \cI_c(\Theta_0)$, we have
  \[
    |\tilde{H}_{i,j}(\Theta_t) - H_{i,j}(\Theta_0)|
    = \frac{1}{L}\sum_{l=0}^{L-1} \left[
    \left\langle \alphal_t \gl_t(\bx_j), \gl_t(\bx_i) \right\rangle
    - \left\langle \gl_0 (\bx_j),\gl_0(\bx_i) \right\rangle \right]
    \lesssim \frac{c^2}{L},
  \]
  and $\lambda_{\min}(\tilde{H}(\Theta_t)) \ge \lambda_{\min}(H(\Theta_{0})) / 2
  = \lambda'/2$ since $L \gtrsim c^2 / \lambda'$. Therefore,  we have
  \begin{equation}
    I_1 \le -\frac{\eta L}{2 n} \cdot \frac{\lambda'}{2}
    \sum_{i=1}^n {[f(\bx_i; \Theta_t) - y_i]}^2
    = -\frac{\lambda' \eta L}{2} \hcR(\Theta_t).
    \label{eqn:d_risk_i1}
  \end{equation}

  For  $I_2$ and $I_3$, since $\Theta_t, \Theta_{t+1} \in
  \cI_{c}$, Lemma~\ref{lemma: upper-bound} implies that $\|\al_{t+1} - \al_t\| = \eta
  \|\nabla_{\al} \hcR(\Theta_t)\| \lesssim \eta \sqrt{\hcR(\Theta_t)}$, and
  Lemma~\ref{lem:discrete_step} implies that $\|\gl_{t+1} - \gl_t\| \lesssim c \eta
  \sqrt{\hcR(\Theta_t)}$.  We have
  \begin{align*}
    I_2 & =\frac{1}{n} \sum_{i=1}^n [f(\bx_i; \Theta_t) - y_i]
    \sum_{l=0}^{L-1}
    \left\langle \al_{t+1}, \gl_{t+1}(\bx_i) - \gl_t(\bx_i) \right\rangle \\
    & \le \frac{1}{n} \sum_{i=1}^n |f(\bx_i; \Theta_t) - y_i|
    \sum_{l=0}^{L-1} \|\al_{t+1}\| \|\gl_{t+1}(\bx_i) - \gl_t(\bx_i)\| \\
    & \lesssim \sqrt{\hcR(\Theta_t)} \cdot L\cdot \frac{c}{L}
    \cdot c \eta \sqrt{\hcR(\Theta_t)} \\
    & = c^2 \eta \hcR(\Theta_t)
    \lesssim \lambda' \eta L \hcR(\Theta_t),
  \end{align*}
  where $c^2 \eta \lesssim \lambda' \eta L $ since $L \gtrsim c^2 /
  \lambda'$. Meanwhile, since
  \begin{align}
    f(\bx_i; \Theta_{t+1}) & - f(\bx_i; \Theta_t)
    = \sum_{l=0}^{L-1} \left[
    \left\langle \al_{t+1} - \al_t, \gl_t(\bx_i) \right\rangle
    + \left\langle \al_{t+1}, \gl_{t+1}(\bx_i) - \gl_t(\bx_i) \right\rangle
    \right] \nonumber \\
    & \le \sum_{l=0}^{L-1} \left[ \|\al_{t+1} - \al_t\| \|\gl_t(\bx_i)\|
    + \|\al_{t+1}\| \|\gl_{t+1}(\bx_i) - \gl_t(\bx_i)\| \right] \nonumber \\
    & \lesssim \sum_{l=0}^{L-1} \left[ \eta \sqrt{\hcR(\Theta_t)} \cdot 1
    + \frac{c}{L} \cdot c \eta \sqrt{\hcR(\Theta_t)} \right] \nonumber \\
    & = \eta(L + c^2) \sqrt{\hcR(\Theta_t)}
    \lesssim \eta L \sqrt{\hcR(\Theta_t)},
    \label{eqn:d_risk_i2}
  \end{align}
 we also have
  \begin{equation}
    I_3
    = \frac{1}{2n} \sum_{i=1}^n {[f(\bx_i; \Theta_{t+1})-f(\bx_i; \Theta_t)]}^2
    \lesssim \eta^2 L^2 \hcR(\Theta_t)
    \lesssim \frac{\lambda' \eta L}{n} \hcR(\Theta_t),
    \label{eqn:d_risk_i3}
  \end{equation}
  where $\eta^2 L^2 \lesssim \lambda' \eta L $ since $\eta \lesssim
  \lambda' / L$.

  Combining~\eqref{eqn:d_risk_i1}~\eqref{eqn:d_risk_i2} and~\eqref{eqn:d_risk_i3},
  we get
  \[
    \hcR(\Theta_{t+1}) - \hcR(\Theta_t) = I_1 + I_2 + I_3
    \le -\frac{\lambda' \eta L}{4 } \hcR(\Theta_t).
  \]
\end{proof}

\begin{lemma}\label{lem:discrete_converge2neighbor}
  Let $c \gtrsim 1/\lambda'$ where $\lambda' = \lambda_{\min}(H(\Theta_0)) >
  0$. Let $t_0$ be such that  $\Theta_t \in \cI_c(\Theta_0)$ and  $\hcR(\Theta_t)
  \le {\left( 1 - \frac{\lambda' \eta L}{4} \right)}^t \hcR(\Theta_0)$ for $t=0,1,\dots,t_0$, then
  we have  $\Theta_{t_{0}+1}\in\cI_{c}(\Theta_0)$.
\end{lemma}

\begin{proof}
  From Lemma~\ref{lemma: upper-bound}, $\|\al_{t+1} - \al_t\| = \eta
  \|\nabla_{\al} \hcR(\Theta_t)\| \lesssim \eta \sqrt{\hcR(\Theta_t)}$
  for $\Theta_t \in \cI_c(\Theta_0)$. So we have
  \begin{align*}
    \|\al_{t_0 + 1}-\al_0\|
    & \le \sum_{t=0}^{t_0} \|\al_{t+1} - \al_t\|
    = \eta \sum_{t=0}^{t_0} \|\nabla_{\al} \hcR(\Theta_t)\| \\
    & \lesssim \eta \sum_{t=0}^{t_0} \sqrt{\hcR(\Theta_t)}
    \le \eta \sqrt{\hcR(\Theta_0)} \sum_{t=0}^{t_0}
    {\left( 1 - \frac{\lambda' \eta L}{4 n} \right)}^\frac{t}{2} \\
    & \le \eta \sqrt{\hcR(\Theta_0)} \cdot \frac{8}{\lambda' \eta L}
    = \frac{8}{\lambda' L} \sqrt{\hcR(\Theta_0)}
    <\frac{c}{L}
  \end{align*}
  if we choose the absolute constant $C$ large enough. Similar results hold
  for $\Bl$, $\Cl$ and $\rl$. Therefore, $\Theta_{t_0 + 1} \in \cI_c(\Theta_0)$.
\end{proof}

\paragraph*{Proof of Theorem~\ref{thm:discrete}}

  Since $L \gtrsim \lambda_n^{-2} \ln(n^2 / \delta)$, by
  Lemma~\ref{lem:lambda0}, $\lambda' = \lambda_{\min} (H(\Theta_0)) \ge
  \lambda_n / 2$ with probability at least $1-\delta$.

  Let $c \gtrsim 1 / \lambda'$. We prove the theorem by introduction:
  \begin{equation}
    \Theta_t \in \cI_c(\Theta_0), \quad \hcR(\Theta_t) \le
    {\left( 1 - \frac{\lambda_n \eta L}{8} \right)}^t \hcR(\Theta_0).
    \label{eqn:discrete_induct}
  \end{equation}
  Obviously statement~\eqref{eqn:discrete_induct} holds for $t = 0$. Assume that
  it holds for $t = 0, 1, \dots, t_0$, then by
  Lemma~\ref{lem:discrete_converge2neighbor}, we have $\Theta_{t_0 + 1} \in
  \cI_c(\Theta_0)$. Thus the assumptions of Lemma~\ref{lem:discrete_neighbor2converge}
  hold, and we have
  \[
    \hcR(\Theta_{t_0 + 1})
    \le \left( 1 - \frac{\lambda' \eta L}{4 } \right) \hcR(\Theta_{t_0})
    \le {\left( 1 - \frac{\lambda_n \eta L}{8} \right)}^{t_0 + 1}
    \hcR(\Theta_0).
  \]
  Therefore, statement~\eqref{eqn:discrete_induct} holds for $t_0 + 1$. This completes the  proof of Theorem~\ref{thm:discrete}.

\section{Conclusion}

\newcommand{\bU}{\bm{U}}
\newcommand{\bW}{\bm{W}}

With what we have learned from numerical results presented 
earlier and the theoretical results for this simplified deep neural network model, 
one is tempted to speculate  that similar results hold for the ResNet model
for the general case, i.e. 
the GD dynamics does converge to the global minima of the empirical risk,
and the GD path stays close to the GD path for the compositional random feature
model during the appropriate time intervals.
A natural next step is to put these speculative statements on a rigorous footing. 

We should also note that even if they are true, these results do not exclude the possibility that
in some other parameter regimes, the models found by the GD algorithm for ResNets
generalize better than the ones for the compositional random feature model
and that there exist non-trivial ``implicit regularization'' regimes.
They do tell us that finding these regimes is a non-trivial task.

\paragraph*{Acknowledgement:} The work presented here is supported in part by a gift to Princeton University from iFlytek and the ONR grant N00014-13-1-0338.

\bibliographystyle{plain}
\bibliography{dl_ref}

\newpage
\appendix
\section{Experiment setup}
\label{sec: res-setup}
\paragraph*{ResNet}
The residual network considered  is given by 
\begin{equation}\label{eqn: resnet}
\begin{aligned}
\bh^{(1)} &= V^{(0)} \bx \in\RR^{d+1} \\
\bh^{(l+1)} &=
\bh^{(l)}  +  U^{(l)} \sigma(V^{(l)}\bh^{(l)}) , \quad l=1, \cdots, L-1\\
f(\bx;\Theta) &= \bw^T \bh^{(L)},
\end{aligned}
\end{equation}
 where $U^{(l)}\in \RR^{(d+1)\times m}, V^{(l)}\in \RR^{m\times (d+1)},\bw\in\RR^{d+1}$, and 
 $V^{(0)} = (I_d, 0)^T\in \RR^{(d+1)\times d}$.  Throughout the experiment, since we are interested in the effect of depth, we choose $m=1$ to speed up the training.  For any $j, k\in [d+1], i\in [m]$, we initialize the ResNet by $U^{(l)}_{k,i}=0$, $V^{(l)}_{i,j}\sim \mathcal{N}(0,\frac{1}{m})$. $\bw$ is initialized as $(0,\dots,0,1)$ and kept fixed during training. We use gradient descent to optimize the empirical risk and the learning rates for the ResNets of different depths are manually tunned to achieve the best test performance.

\paragraph*{Compositional norm regularization}
Consider the following regularized model
\begin{equation}\label{eqn: path-norm-reg}
    \text{minimize}\,\, J(\Theta):= \hat{\cR}_n(\Theta) + \frac{\lambda}{\sqrt{n}} \|\Theta\|_{\cP}.
\end{equation}
Here the compositional norm is defined  by (in this case  it is the direct extension
of the path norm \cite{neyshabur2015norm} to ResNets )
\[
    \|\Theta\|_{\cP} \Def |\bw|^T\prod_{l=1}^L (I+|U^{(l)}||V^{(l)}|)|V^{(0)}|.
\]
Here  for a matrix $A=(a_{i,j})$, we define $|A|\Def(|a_{i,j}|)$.
For the above regularized model, we use Adam~\cite{kingma2014adam} optimizer to solve problem~\eqref{eqn: path-norm-reg}.

\section{Proof of Lemma~\ref{lem:lambda0}}
\label{sec: init-Gram-matrix}
\begin{proof}
For a given $t\geq 0$, define the event $S_{i,j}=\{\Theta_0: |H_{i,j}(\Theta_0)-k(\bx_i,\bx_j)|\leq t/n\}$.
Using Hoeffding's inequality, we get  $\PP\{S_{i,j}^c\}\leq e^{-2mLt^2}$. Hence
\begin{align*}
\PP\{\cap_{i,j} S_{i,j}\} &= 1- \PP\{\cup_{i,j} S^c_{i,j}\} \\
                        & \geq 1 - \sum_{i,j} \PP\{S_{i,j}^c\} \\
                        & \geq 1 - n^2 e^{-2mL t^2}.
\end{align*}
Therefore, with probability $1 - n^2 e^{-2mL t^2}$, the following inequality holds,
\[
\lambda_{\min}(H(\Theta_0))\geq \lambda_{\min}(K) - \|H-K\|_F \geq \lambda_n - t.
\]
Taking $t=\lambda_n/4$, we complete the proof.
\end{proof}

\section{Proofs for the random feature model}
\subsection{Proof of Theorem~\ref{thm: approx}}
\label{sec: random-feature}
Note that $B_0 = (\bb^0_1,\dots,\bb^0_{mL})^T \in \RR^{mL\times d}$ with $\sqrt{m}\bb^0_j\sim \pi(\cdot)$.
By choosing $\ba^* = (a_1^*,\dots,a^*_{mL})^T$ with $a_j^*=a^*(\sqrt{m}\bb_j^0)/(\sqrt{m}L)$, we have
\[
    \tilde{f}(\bx;a^*,B_0) = \frac{1}{mL}\sum_{j=1}^{mL}a^*(\sqrt{m}\bb_j^0)\sigma(\bx^T\sqrt{m}\bb_j^0).
\]
Therefore $\EE_{B_0}[\tilde{f}(\bx;a^*,B_0)] = f^*(\bx)$. Consider the following random variable,
\[
Z(B_0) = \|\tf(\cdot;\ba^*,B_0)-f^*(\cdot)\| \Def\sqrt{\EE_{\bx}|\tf(\bx;\ba^*,B_0)-f^*(\bx)|^2}.
\]
Let $\tB_0=(\bb^0_1,\dots,\tilde{\bb}^0_j,\dots,\bb^0_{mL})^T$, which equals to $B_0$ except $\tilde{\bb}_j^0$. Then we have
\begin{align*}
|Z(B_0)-Z(\tB_0)| &= \|\tf(\cdot;\ba^*(B_0),B_0)-f^*(\cdot)\| - \|\tf(\cdot;\ba^*(\tB_0),\tB_0)-f^*(\cdot)\| \\
&\leq \|\tf(\cdot;\ba^*(B_0),B_0) - \tf(\cdot;\ba^*(\tB_0),\tB_0)\| \\
&\leq \frac{2\gamma({f^*})}{mL}
\end{align*}
Applying McDiarmid's inequality, we have that with probability $1-\delta$ the following inequality holds
\begin{equation}\label{eqn: l1-prob-bound}
Z(B_0)\leq \EE[Z(B_0)] + \gamma({f^*}) \sqrt{\frac{2\ln(1/\delta)}{mL}}.
\end{equation}
Using the Cauchy-Schwarz inequality, we have  $\EE[Z(B_0)]\leq \sqrt{\EE[Z^2(B_0)]}$. Since
\begin{align*}
    \EE[Z^2(B_0)] &= \EE_{\bx} \EE_{B_0} |\tf(\bx;\ba^*,B_0)-f^*(\bx)|^2 \leq \frac{\gamma(f^*)}{mL},
\end{align*}
 we obtain $\EE[Z(B_0)]\leq \frac{\gamma(f^*)}{\sqrt{mL}}$.
 Plugging this into \eqref{eqn: l1-prob-bound} gives us
\[
\EE_{\bx} |\tf(\bx;\ba^*,B_0)-f^*(\bx)|^2 = Z^2(B_0) \leq \frac{\gamma^2(f^*)}{mL}\left(1+\sqrt{2\ln(1/\delta)}\right)^2.
\]
In addition, it is clear that $\|\ba^*\| = \sqrt{\sum_{j=1}^{mL}|a_j^*|^2}\leq \gamma(f^*)/\sqrt{L}$.

\subsection{Proof of Lemma~\ref{lem: convergence-reference-dynamics}}
\label{sec: rdf-2}
Let $J(t)\Def t (\hat{\cE}_n(\tilde{\ba}_t;B_0) - \hat{\cE}_n(\ba^*;B_0)) + \frac{1}{2}\|\tilde{\ba}_t-\ba^*\|^2$, then we have
\[
    \frac{ dJ(t)}{dt} = \hat{\cE}_n(\tilde{\ba}_t; B_0) - \hat{\cE}_n(\ba^*;B_0) - \langle \tilde{\ba}_t-\ba^*, \nabla \hat{\cE}_n(\ta_t)\rangle  - t \|\nabla \hat{\cE}_n(\ta_t;B_0)\|^2.
\]
By using the convexity of $\hat{\cE}_n(\cdot;B_0)$, it is easy to see that $dJ/dt\leq 0$. Therefore, $J(t)\leq J(0)$. This leads to
\begin{align*}
 t (\hat{\cE}_n(\tilde{\ba}_t;B_0) - \hat{\cE}_n(\ba^*;B_0)) + \frac{1}{2}\|\tilde{\ba}_t-\ba^*\|^2 \leq \frac{1}{2}\|\ba_0-\ba^*\|^2.
\end{align*}
It is easy to see that
\begin{align*}
 t \hat{\cE}_n(\tilde{\ba}_t;B_0)  & \leq  t\hat{\cE}_n(\ba^*;B_0) +  \frac{1}{2}\|\ba_0-\ba^*\|^2 - \frac{1}{2}\|\tilde{\ba}_t-\ba^*\|^2\\
\frac{1}{2}\|\tilde{\ba}_t-\ba^*\|^2  &\leq \frac{1}{2}\|\ba_0-\ba^*\|^2 + t \hat{\cE}_n(\ba^*;B_0).
\end{align*}
This completes the proof.

\subsection{Proof of Theorem~\ref{thm: generalization-random-feature}}
\label{sec: rdf-3}
Let 
\[
\gen(\ba;B_0)=|\cE(\ba;B_0)-\hat{\cE}_n(\ba;B_0)|.
\]
Using Lemma~\ref{lem: convergence-reference-dynamics} and Theorem~\ref{thm: approx}, we have
\begin{align}\label{eqn: rdf-decomp}
\nonumber \cE(\ta_t;B_0) &\leq |\cE(\ta_t;B_0)-\hat{\cE}_n(\ta_t;B_0)| + \hat{\cE}_n(\ta_t;B_0) \\
\nonumber &\stackrel{(i)}{\le} |\cE(\ta_t;B_0)-\hat{\cE}_n(\ta_t;B_0)| + |\hat{\cE}_n(\ba^*; B_0)- \cE(\ba^*;B_0)| + \cE(\ba^*;B_0) + \frac{\|\ta_0-\ba^*\|^2}{2t} \\
&\stackrel{(ii)}{\le} \text{gen}(\ta_t;B_0) + \text{gen}(\ba^*;B_0) + \frac{c(\delta) \gamma^2(f^*)}{mL}+ \frac{\gamma^2(f^*)}{2Lt}.
\end{align}
We next proceed to  estimate the two generalization gaps by using the Rademacher complexity.

Let $\cF_c = \{\ba^T\sigma(B_0\bx) : \|\ba\|\leq c\}$,  and $\cH_c = \{\frac{1}{2}(\ba^T\sigma(B_0\bx)-y)^2 :  \|\ba\|\leq c \}$. We first have
\begin{align*}
\rad(\cF_c) & =\frac{1}{n}\EE [\sup_{\|\ba\|\leq c} \sum_{i=1}^n \xi_i \ba^T\sigma(B_0\bx_i)] \\
&= \frac{1}{n}\EE [\sup_{\|\ba\|\leq c} \langle \ba,\sum_{i=1}^n \xi_i \sigma(B_0\bx_i)\rangle]\\
&\leq \frac{c}{n}\EE [\|\sum_{i=1}^n \xi_i \sigma(B_0\bx_i)\|] \\
&\stackrel{(i)}{\leq} \frac{c}{n}\sqrt{\EE [\|\sum_{i=1}^n \xi_i \sigma(B_0\bx_i)\|^2]} \\
&\leq \frac{c}{n}\sqrt{\sum_{i=1}^n \EE_{}[\xi_i^2]\sigma(B_0\bx_i)^T\sigma(B_0\bx_i) + \sum_{i\neq j}\EE[\xi_i\xi_j]\sigma(B_0\bx_i)^T\sigma(B_0\bx_j)} \\
&\leq \frac{\sqrt{L}c}{\sqrt{n}},
\end{align*}
where  the expectation is taken over $\xi_1,\dots,\xi_n$, and $(i)$ follows from the fact that $\sqrt{t}$ is concave in $(0,\infty]$.   

For any $\ba$ that satisfies $\|\ba\|\leq c$, it is obvious that $|\tilde{f}(\bx;\ba,B_0)|\leq \sqrt{L} c$. Let $\ell(y',y)=\frac{(y'-y)^2}{2}$ denote the loss function. In this case, $|y'|\leq \sqrt{L}c$ and $|y|\leq 1$, so $\ell(\cdot,y)$ is $(\sqrt{L}c+1)-$Lipschitz continuous. Using the contraction property of Rademacher complexity (see e.g. \cite{shalev2014understanding,mohri2018foundations}), we have
\[
\rad(\cH_c)\leq (\sqrt{L}c+1)\rad(\cF_c) \leq \frac{\sqrt{L}c(c+1)}{\sqrt{n}}.
\]
The standard Rademacher complexity based bound gives that  for any fixed $\delta\in (0,1)$, with probability at least $1-\delta$
\begin{equation}\label{eqn: apd-posteriori-bound}
\gen(\ba;B_0) \leq  \frac{2(\sqrt{L}c+1)c\sqrt{L}}{\sqrt{n}} + (\sqrt{L}c+1)^2\sqrt{\frac{\ln(1/\delta)}{n}},
\end{equation}
for any $\ba$ such that $\|\ba\|\leq c$.

Consider a decomposition of the whole hypothesis space, $\cF=\cup_{k\in \NN^{+}} \cF_k$ with $\cF_k = \{\ba^T\sigma(B_0\bx) : \|\ba\|\leq c\}$. Let $\delta_{k}=\frac{\delta}{k^2}$. If $k$ is pre-specified, then with probability $1-\delta_k$, \eqref{eqn: apd-posteriori-bound} holds for $c=k$. Then from  the union bound for all $k\in \NN^{+}$, we obtain that for any fixed $\delta>0$,  with probability $1-\delta$, the following estimates holds for any $\ba$,
\[
\gen(\ba;B_0) \lesssim  \frac{(\sqrt{L}\|\ba\|+1)\|\ba\|\sqrt{L}}{\sqrt{n}} + (\sqrt{L}\|\ba\|+1)^2\sqrt{\frac{\ln((1+\|\ba\|)/\delta)}{n}}.
\]
Theorem~\ref{thm: approx} says that $\|\ba^*\|\leq \gamma(f^*)/\sqrt{L}$. Hence we obtain
\begin{equation}\label{eqn: fixed-a}
\gen(\ba^*;B_0)\lesssim  \gamma^2(f^*) \frac{1+\sqrt{\ln(1/\delta)}}{\sqrt{n}}.
\end{equation}

By Lemma~\ref{lem: convergence-reference-dynamics}, we have
\begin{equation*}
\begin{aligned}
\|\ta_t\| & \leq 2\|\ba^*\| + t \hat{\cE}_n(\ba^*;B_0) \\
&\leq 2\|\ba^*\| + t\left(\cE(\ba^*;B_0) + \gen(\ba^*;B_0)\right) \\
&\lesssim \frac{\gamma(f^*)}{\sqrt{L}} + t \gamma^2(f^*) (\frac{c^2(\delta)}{mL} + \frac{1+\sqrt{\ln(1/\delta)}}{\sqrt{n}}),
\end{aligned}
\end{equation*}
where $c(\delta)=1+\sqrt{\ln(1/\delta)}$.
So using $L\geq \gamma^2(f^*)\geq 1$, we have for any $t\in [0,1]$,
\begin{align}\label{eqn: dynamics-a}
  \gen(\ta_t;B_0) & \lesssim \left(1 + \frac{\sqrt{L}\gamma(f^*)t}{\sqrt{n}}\right)^2 \frac{c^3(\delta)\gamma^2(f^*)}{\sqrt{n}},
\end{align}
Plugging the estimates \eqref{eqn: fixed-a} and \eqref{eqn: dynamics-a} into \eqref{eqn: rdf-decomp} completes the proof.
\end{document}